\documentclass[letterpaper]{article} % DO NOT CHANGE THIS

\makeatletter
\let\AAAIOriginalAddContentsLine\addcontentsline
\makeatother

\usepackage[preprint]{aaai2027}  % DO NOT CHANGE THIS
\usepackage{amsthm}

\usepackage{times}  % DO NOT CHANGE THIS
\usepackage{helvet}  % DO NOT CHANGE THIS
\usepackage{courier}  % DO NOT CHANGE THIS
\usepackage[hyphens]{url}  % DO NOT CHANGE THIS
\usepackage{titletoc}

\usepackage{graphicx} % DO NOT CHANGE THIS
\usepackage{natbib}  % DO NOT CHANGE THIS AND DO NOT ADD ANY OPTIONS TO IT
\usepackage{caption} % DO NOT CHANGE THIS AND DO NOT ADD ANY OPTIONS TO IT
\usepackage{booktabs} % commands to create good-looking tables
\usepackage[misc]{ifsym}
\usepackage{tikz} % nice language for creating drawings
\usepackage{mathtools}
\usepackage{amssymb}

\usepackage{algorithm}
\usepackage{algorithmic}
\usepackage{tabularx}

\usepackage{newfloat}
\usepackage{listings}
\DeclareCaptionStyle{ruled}{labelfont=normalfont,labelsep=colon,strut=off} % DO NOT CHANGE THIS
\floatstyle{ruled}
\newfloat{listing}{tb}{lst}{}
\floatname{listing}{Listing}

\newcommand{\x}[1]{\mathbf{x} #1}
\newcommand{\y}[1]{\mathbf{y} #1}
\newcommand{\z}[1]{\mathbf{z} #1}
\newcommand{\btheta}[1]{\mathbf{\theta} #1}
\newcommand{\ubar}[1]{\text{\b{$#1$}}}
\DeclareMathOperator{\dom}{dom}
\DeclareMathOperator{\dist}{dist}
\DeclareMathOperator{\prox}{prox}
\DeclareMathOperator*{\argmin}{argmin}
\DeclareMathOperator*{\minimize}{minimize}
\newcommand{\indicator}{\mathbf{1}}
\newcommand{\papertablesize}{\small}
\newcommand{\papertablecaption}[1]{\samepage\captionof{table}{#1}}
\newenvironment{papertableblock}
  {\par\begin{center}\begin{minipage}{\textwidth}\centering}
  {\end{minipage}\end{center}}
\newcommand{\iidresultsfigure}{%
  \includegraphics[width=0.9\linewidth]{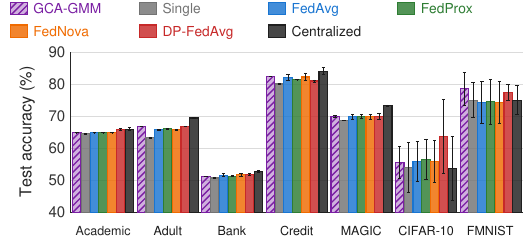}}
\newcommand{\leakageNTMSEFigure}{%
  \includegraphics[width=0.9\linewidth]{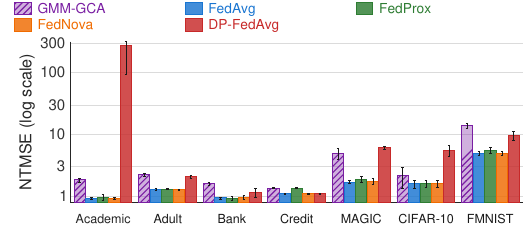}}
\newcommand{\leakageCosineFigure}{%
  \includegraphics[width=0.9\linewidth]{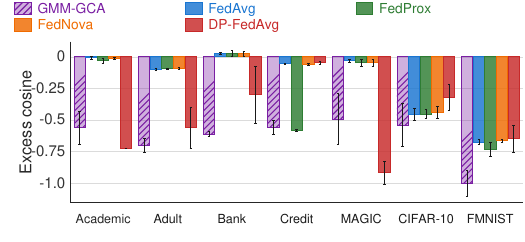}}
\newtheorem{proposition}{Proposition}
\newtheorem{theorem}{Theorem}
\newtheorem{lemma}{Lemma}

\title{GCA: Global Centroid Alignment in Federated Learning}
\author{
    \textbf{Jong-Ik Park}\textsuperscript{\rm 1},
    Harry Jiang\textsuperscript{\rm 1},
    Logan Blakely\textsuperscript{\rm 2},
    Georgios Fragkos\textsuperscript{\rm 2},
    Shamina Hossain-McKenzie\textsuperscript{\rm 2},
    Carlee Joe-Wong\textsuperscript{\rm 1}
}

\affiliations{
    \textsuperscript{\rm 1}Carnegie Mellon University\\
    \textsuperscript{\rm 2}Sandia National Laboratories
}

\begin{document}

\maketitle

\begin{abstract}
Autoencoder (AE)-based federated learning (FL) is attractive for anomaly detection when clients have limited local data. However, conventional FL exchanges AE parameters or gradients, incurring substantial communication overhead and potentially exposing input training data information, since AEs are explicitly optimized to reconstruct their inputs.
We introduce \emph{Global Centroid Alignment (GCA)}, a latent-code-mediated FL protocol that coordinates clients without transmitting AE parameters or gradients. In each round,
(1) clients first train their local AEs using a \emph{reconstruction} update and upload a small subset of encoder latent codes to the FL server.
(2) The server pools these codes, fits a clustering model, and broadcasts only \emph{global latent centroids and their support counts}.
(3) Each client then updates its encoder by aligning its local latent codes with the \emph{nearest} centroid using \emph{inverse-count} weighting to emphasize globally underrepresented patterns.
Steps (1)--(3) repeat over communication rounds.
Because GCA exchanges only sampled latent codes and centroid statistics, its communication cost depends on latent dimensionality and the numbers of uploaded codes and returned centroids rather than on AE model size. Across five tabular and two vision benchmarks, GCA yields higher reconstruction error under a server-side client data extraction attack in all 21 comparisons and clearly lower cosine similarity in 20 of 21 comparisons with FedAvg, FedProx, and FedNova, showing its ability to protect training data. It even improves test accuracy over FedAvg by up to $5.76\%$. GCA achieves extraction defense comparable to DP-FedAvg, remains effective when DP-FedAvg does not reduce target resemblance, and lowers per-round communication by up to $99.15\%$.
\end{abstract}

\section{Introduction}
\label{sec:intro}
Federated learning (FL) enables multiple clients to train their own \emph{globally informed} models \emph{without moving their data} to another location~\cite{mcmahan2017communication, li2020federated, park2024federated}.
Each client updates its model locally and sends an update (e.g., gradients or parameter changes) to a coordinating server, which aggregates these updates to form a global model and broadcasts the global model back; clients then continue local adaptation~\cite{konecny2016federated}. This process repeats until convergence (or for a fixed number of rounds)~\cite{li2020review, zhang2021survey}.
Because raw data never leaves client devices, FL reduces data movement compared to centralized training and limits exposure of sensitive records, improving privacy by design~\cite{li2021survey, thapa2021advancements}. As a result, it is widely deployed in healthcare, finance, mobile keyboards, and industrial IoT, and is especially valuable in \emph{data-scarce} and \emph{data-sensitive} settings where cross-client structure can be exploited without sharing raw data~\cite{long2020federated, rieke2020future}.

A common task in these applications is that of detecting anomalous data, e.g., device telemetry, medical monitoring, predicting faults in industrial IoT settings, and fraud/abuse detection---where the learning task is to detect rare yet high-impact anomalies with low false-alarm rates~\cite{kea2023enhancing, novoa2023fast, shrestha2024anomaly}. %predicting future faults in industrial IoT settings~\cite{ma2025fedsum, khowaja2025selffed}. 
For such tasks, FL commonly employs reconstruction models such as \emph{autoencoders (AEs)}~\cite{jiang2022privacy, vucovich2023anomaly}. 
%
% These settings often have few labeled examples, noisy labels, and \emph{imbalanced} class distributions~\cite{ma2025fedsum, khowaja2025selffed}. In such cases, the objective shifts to modeling ``normal'' behavior and flagging deviations, so FL commonly employs reconstruction models such as \emph{autoencoders (AEs)}~\cite{jiang2022privacy, vucovich2023anomaly}. 
%
To do so, each client trains an AE on its own data (which is assumed to be normal, i.e., non-anomalous) and, in standard practice, uploads parameters or gradients of each locally trained AE model to a central server for aggregation. %This pattern is pervasive in \emph{anomaly detection}---
However, AE-based parameter sharing introduces \textbf{communication} and \textbf{unique privacy challenges}.

For \textbf{communication,} repeatedly \emph{uploading full AE parameters is costly} for low-spec devices (e.g., microcontrollers with limited RAM and sub-Mbps uplinks) and for edge deployments (e.g., gateways with narrow, skewed traffic)~\cite{singh2019detailed, gao2020end}.
On \textbf{privacy,} because an AE is optimized to \emph{reproduce its inputs}, its parameters (and even gradients) can encode fine-grained signatures of local samples; a curious or compromised server may exploit these signals to approximate training records~\cite{suganuma2018exploiting,chen2023auto,ghoshal2025reverse}.
Moreover, the encoder–decoder model often collapses onto a narrow manifold that closely interpolates the training set; as a result, decoding arbitrary latent vectors---or even passing random inputs through the AE---tends to yield outputs resembling the training data~\cite{steck2020autoencoders}. The risk is amplified when a client has few examples (e.g., a small clinic with only a few hundred normal patient records per week), where overfitting makes the autoencoder mapping highly specific to the observed data~\cite{chen2017outlier}.

We therefore propose \textbf{Global Centroid Alignment (GCA)}, an FL protocol for \underline{AE-based anomaly detection} that (i) never shares raw data, parameters, or gradients; (ii) transfers sampled latent codes and centroid statistics; and (iii) remains stable on normals-only data.
\textbf{Reconstruction.} GCA begins with a local training: clients train as in standard AE learning, jointly updating the encoder and decoder so that the decoder's output closely matches the input~\cite{tschannen2018recent}. Clients then transmit only a subset of \emph{latent codes} from their training data (not parameters or raw samples).
\textbf{Aggregation.} The server performs \emph{aggregation} by fitting a lightweight clustering model (e.g., a Gaussian mixture~\cite{mclachlan2014number}) on the pooled latents, and broadcasts only the resulting \emph{global latent centroids} and their \emph{support counts}, not parameters or gradients, to the clients.
\textbf{Alignment.} Finally, on the client side, an \emph{alignment} step pulls encoder outputs (latent codes) toward the \emph{nearest} latent centroid using an \emph{inverse-count} weight, which emphasizes globally underrepresented patterns and discourages collapse to dominant modes~\cite{anand2010approach, mohammed2020machine}.
 \\
%
%
% This reconstruction$\rightarrow$aggregation$\rightarrow$alignment schedule first allows reconstruction to adapt locally and then injects cross-client structure, reducing overfitting to site-specific idiosyncrasies. Inverse-count weighting in the alignment step further improves coverage of rare or previously unseen modes~\cite{anand2010approach, chang2013oversampling, li2020analyzing}. Because clients send only a fraction of compressed latents and the server returns a small set of prototype means (and counts), per-round communication scales with the number of prototypes rather than model size; the shared information consists of aggregated statistics, not parameters---lowering privacy exposure, especially for large AEs~\cite{yamazaki2022deep, jia2024generative}.
\indent Our key \textbf{contributions} in this paper consist of: \\
$\bullet$ \textbf{GCA.} A communication-efficient, latent-code-mediated FL method in which clients and a server exchange sampled latent codes, global centroids, and support counts instead of model parameters or gradients. \\
$\bullet$ \textbf{Empirical evaluation.} Across five tabular and two vision benchmarks, GCA achieves up to $5.76\%$ higher test accuracy than FedAvg and comparable accuracy to FedProx~\cite{li2020federated}, FedNova~\cite{wang2020tackling}, and DP-FedAvg~\cite{abadi2016deep,mcmahan2018learning}. GCA also provides stronger extraction defense than FedAvg, FedProx, and FedNova, and more consistent defense than DP-FedAvg, which exhibits failure cases in certain settings, while reducing per-round communication by $84.14\%$--$99.15\%$. \\
$\bullet$ \textbf{Theoretical analysis.} We analyze the round-to-round stability of GCA's reconstruction and alignment updates and show that, under our threat model, the server's latent-only observations do not uniquely determine the training records.

%, while reducing communication by up to \textit{zzz}$\times$. 

In the following sections, the \emph{Related Work} reviews prior studies, the \emph{Methodology} introduces GCA, and the \emph{Adversarial Attack Scenarios} describes the considered threat models. The \emph{Experimental Evaluation} presents the experimental results and compares GCA's communication efficiency with standard FL methods. Finally, the \emph{Conclusion} summarizes our findings and outlines directions for future work.

%, and offer theoretical support for our design in Section~\ref{sec: analysis}

\section{Related Work}
\label{sec:related}
%\paragraph{Autoencoders for unsupervised anomaly detection.}
An autoencoder consists of an encoder $E:\mathbb{R}^{d_x}\to\mathbb{R}^{d_z}$ and a decoder $D:\mathbb{R}^{d_z}\to\mathbb{R}^{d_x}$ trained to reconstruct inputs $\mathbf{x}$ by minimizing a reconstruction loss on \emph{normal} data~\cite{jiang2022privacy, vucovich2023anomaly}.
After training, a sample is scored by its reconstruction error $s(\x)=\|\x-D(E(\x))\|^2$, with the assumption that anomalies are poorly reconstructed and thus receive higher scores. In practice, anomaly decisions are often made by thresholding using a percentile of the \emph{training} (or validation) error distribution to control false positives without labels. Common choices such as the 75th or 95th percentiles provide conservative cutoffs under class imbalance and concept drift~\cite{ibidunmoye2017adaptive, van2021anomaly, kiet2025statistical}.
Recent work also suggests that anomaly detection can benefit from \emph{explicitly structuring} the latent space beyond reconstruction alone (e.g., encouraging normal embeddings to concentrate in a compact region~\cite{zong2018deep,zhou2021vae}).
This paradigm is attractive when data and labels are scarce or unavailable, and has been applied broadly to tabular, time-series, and vision data~\cite{cheng2021improved, ma2025fedsum, khowaja2025selffed}.

%\paragraph{AE-based FL applications.}
To collaborate without sharing raw data, clients using AEs can adopt FL with trusted partners while keeping data local, thereby avoiding the costs and risks of centralizing all records~\cite{kea2023enhancing, novoa2023fast, shrestha2024anomaly}. In conventional (e.g., FedAvg-style~\cite{mcmahan2017communication, li2020federated}) FL, clients train local AEs and periodically share model parameters or gradients for aggregation.
%\paragraph{Limitations: communication overhead and privacy risk.}
However, sharing AE parameters or gradients can incur substantial communication overhead~\cite{singh2019detailed,gao2020end} and still risks exposing private information~\cite{suganuma2018exploiting,steck2020autoencoders,chen2023auto,ghoshal2025reverse}. 

% \paragraph{Positioning of GCA.}
Differentially private FL limits the influence of individual training records by clipping and perturbing model updates, but it does not directly remove an AE's underlying objective of reconstructing its inputs~\cite{abadi2016deep,mcmahan2018learning}. \emph{GCA instead avoids sharing even perturbed reconstructive AE parameters or gradients.}
GCA may also appear similar to  federated knowledge-sharing methods that exchange class-conditioned logits, soft labels, or feature prototypes (e.g., FedProto)~\cite{sattler2020communication,tan2022fedproto}. These methods rely on predefined class semantics, whereas \emph{GCA operates on unlabeled normal data, discovers global latent groups from uploaded latent codes,} and aligns local representations with the resulting centroids.
Appendix~\ref{app:gca-vs-prototype-fl} and Table~\ref{tab:gca-vs-prototype-fl} elaborate GCA's distinction from class-prototype FL and explain why withholding an AE reconstruction map as in GCA is complementary to, rather than a replacement for, differential privacy.

\section{Methodology}\label{sec:method}
\begin{figure*}
    \centering
    \includegraphics[width=0.75\linewidth]{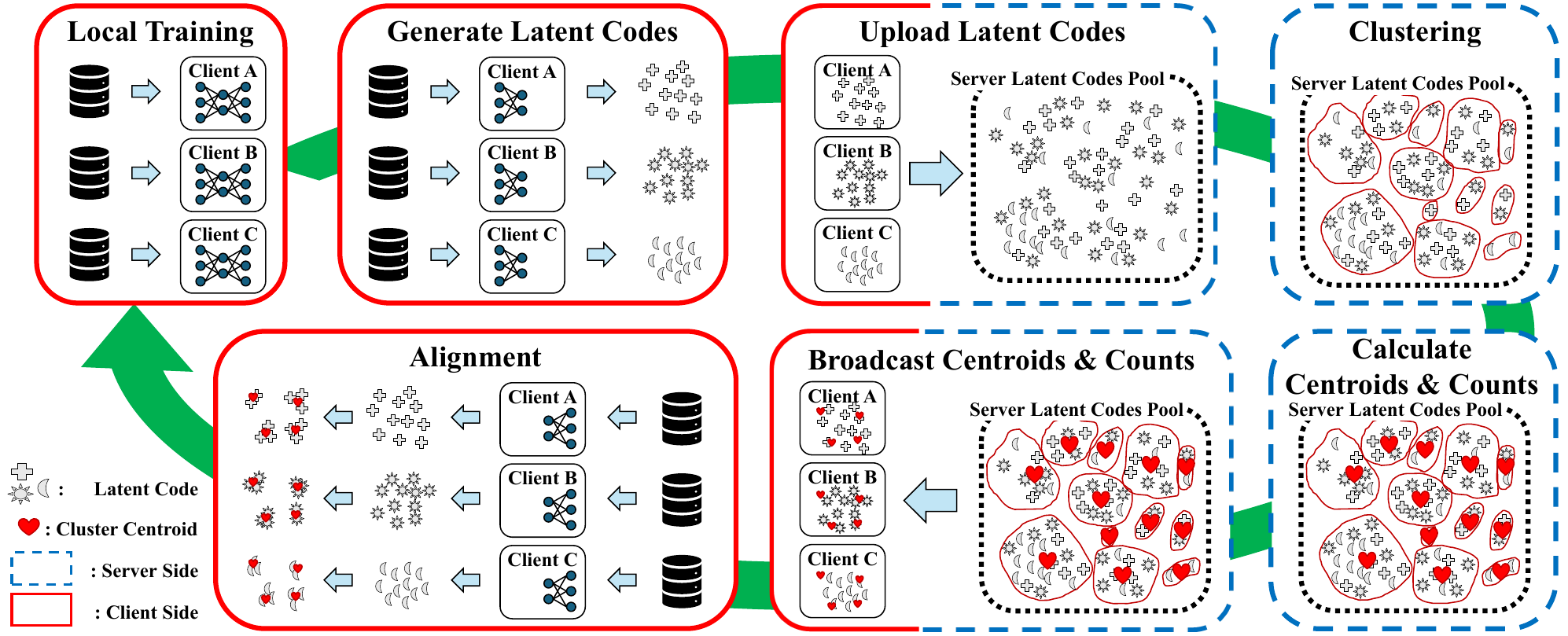}
    \vspace{-0.2 cm}
    \caption{\textbf{GCA pipeline.} (1) \emph{Local training:} each client trains an autoencoder on local normal data. (2) \emph{Generate/Upload Latent upload:} clients generate and upload a subset of latent codes. (3) \emph{Server aggregation:} the server pools latents, performs clustering, and computes global latent centroids and support counts. (4) \emph{Broadcast:} the server sends centroids and support counts to participating clients. (5) \emph{Alignment:} clients align their encoders to the global latent centroids using the broadcast statistics, improving cross-client consistency while keeping data and model updates private.}
    \label{fig:gpa_overview}
\end{figure*}
In this section, we propose a new federated mechanism, \textbf{Global Centroid Alignment (GCA)}, that (i) injects \emph{cross-client} information into each encoder \emph{without sharing raw data, model parameters, or gradients}, (ii) respects privacy requirements and the lack of anomalous training data, and (iii) remains lightweight and stable. Figure~\ref{fig:gpa_overview} summarizes the workflow, and Appendix~\ref{app:implementation-details} (Algorithm~\ref{alg:gpa}) formalizes a complete round.

\paragraph{\textbf{Setup.}}
We consider $C$ clients (sites with local data), indexed by $i\in\{1,\dots,C\}$, each holding \emph{only local normal} data
$\mathcal{D}_i=\{\x_n^{(i)}\in\mathbb{R}^{d_x}\}_{n=1}^{N_i}$. 
At round $r\in\{1,\dots,R\}$, the server samples a subset of participating clients $\mathcal{S}_r\subseteq\{1,\dots,C\}$ with $|\mathcal{S}_r|=m_r$.
Only clients in $\mathcal{S}_r$ perform local updates and communicate with the server in round $r$ (\emph{line 3} in Algorithm~\ref{alg:gpa}).  We denote the participation rate by $\phi:=m_r/C$ (constant for simplicity, but can vary by round).
Training runs for $R$ global rounds; in each round, client $i$ performs $E$ local epochs on minibatches $B_i=\{\x_b\}_{b=1}^{B}$.
For a sample $\x_b$, we define client $i$'s encoder and decoder 
\begin{equation}
\z_b \triangleq E_i(\x_b)\in\mathbb{R}^{d_z}, \; \hat{x}_b \triangleq D_i(\z_b)\in\mathbb{R}^{d_x} \; (d_z\ll d_x).
\end{equation}
Each global round consists of three phases and repeats until convergence or a target metric is achieved. We allocate the total local epochs $T$ between the reconstruction and alignment phases, $T = T_{\text{recon}} + T_{\text{align}}$.
In practice, one can emphasize local stabilization early and cross-client alignment later (e.g., start with larger $T_{\text{recon}}$ and gradually increase $T_{\text{align}}$), but any pair satisfying the sum constraint is admissible. In the below, all line numbers refer to Algorithm~\ref{alg:gpa}.

\subsection{Phase 1: Reconstruction \\ (Client-side, Encoder+Decoder)}\label{reconstruction}
For $T_{\text{recon}}$ epochs, client $i$ minimizes
\begin{equation}
\mathcal{L}^{(i)}_{\text{recon}}(\btheta_{E_i},\btheta_{D_i} ;B_i)=\frac{1}{B}\sum_{b=1}^{B}\big\|\x_b - D_i(E_i(\x_b))\big\|_2^2,
\label{eq:recon_loss}
\end{equation}
with joint updates
\begin{align*}
&\btheta_{E_i}\leftarrow\btheta_{E_i}-\eta_{\text{recon}}\nabla_{\btheta_{E_i}}\mathcal{L}^{(i)}_{\text{recon}}, 
\\
&\btheta_{D_i}\leftarrow\btheta_{D_i}-\eta_{\text{recon}}\nabla_{\btheta_{D_i}}\mathcal{L}^{(i)}_{\text{recon}} \qquad \text{\emph{(lines 6-13)}}.
\end{align*}
Starting with local reconstruction lets each AE \emph{stabilize} on its own manifold before any global pressure is applied.
After the local training, each participating client $i\in\mathcal{S}_r$ uploads only a \emph{subset} of latent codes
\begin{equation*}
\z_i^{\uparrow}\subset\{E_i(\x)\,:\,\x\in\mathcal{D}_i\},\quad |\z_i^{\uparrow}|=\rho_i N_i, \quad \rho_i\in(0,1],
\end{equation*}
to the server, \emph{not} raw data and \emph{not} model parameters or gradients \emph{(lines 15-19)}.

\subsection{Phase 2 --- Aggregation \\ (Server-side, on Pooled Latents)}\label{aggregation}
After Phase 1 in round $r$, the server pools the uploaded latents in round $r$ as
$Z^{(r)}=\bigcup_{i\in\mathcal{S}_r}\z_i^{\uparrow}=\{\z_n\}_{n=1}^{N_r}$,
where $N_r=\sum_{i\in\mathcal{S}_r}|\z_i^{\uparrow}|$,
and fits a clustering model to $Z$ (e.g., hard/soft $K$-means or a $K$-component Gaussian mixture model (GMM)) \emph{(lines 21-24)}.
\paragraph{\textbf{Example: GMM Component Means as Global Centroids.}}
For a GMM with mixture weights $\{\pi_k\}_{k=1}^{K}$, means $\{\mathbf{\mu}_k\}_{k=1}^{K}$, and covariances $\{\Sigma_k\}_{k=1}^{K}$,
the expectation--maximization (EM) responsibilities, effective counts, and the component mean are
$
\gamma_{nk}=\frac{\pi_k\,\mathcal{N}(\z_n\!\mid\!\mathbf{\mu}_k,\Sigma_k)}{\sum_{j=1}^{K}\pi_j\,\mathcal{N}(\z_n\!\mid\!\mathbf{\mu}_j,\Sigma_j)}$, 
$
N_k=\sum_{n=1}^{N_r}\gamma_{nk}$, and 
$ 
\mathbf{\mu}_k=\frac{1}{N_k}\sum_{n=1}^{N_r}\gamma_{nk}\,\z_n$ respectively.
In this GMM example, $N_k$ is a soft count because it sums responsibilities rather than hard assignments.
Other clustering methods similarly yield latent centroids $\{\mathbf{\mu}_k\}$ and associated counts.
The server then broadcasts the \emph{global latent centroid set} $\mathcal{M}=\{\mathbf{\mu}_k\}_{k=1}^{K}$ and \emph{counts} $\mathbf{N}_{\mathcal{M}}=\{N_k\}_{k=1}^{K}$ to each participating client $i\in\mathcal{S}_r$.

Working with pooled \emph{latent codes} avoids parameter and gradient exchange and reduces uplink cost, while clustering distills cross-client structure into a compact set of global latent centroids that clients can use without accessing other clients' data or models.
The implementation broadcasts all $K$ returned centroids and their support counts to the clients.

\subsection{Phase 3 --- Alignment (Client-side, Encoder-only)}\label{alignment}
For $T_{\text{align}}$ epochs, each client $i$ nudges its encoder toward the broadcast latent centroids while \emph{emphasizing rare global patterns}.
For each $\x_b$,
\begin{equation}
    a_b\in\argmin_{k\in\{1,\dots,K\}}\ \|E_i(\x_b)-\mathbf{\mu}_k\|_2,
\label{eq:nearest_latent}
\end{equation}
and we form an \emph{inverse-count} weight with approximately \emph{mean one} across centroids to keep the loss scale stable \emph{(lines 27-28)}. Let $\widehat N_k\triangleq\max(N_k,\epsilon)$, $\bar N\triangleq\frac{1}{K}\sum_{k=1}^{K}\widehat N_k$, and define
\begin{equation}
\begin{aligned}
\tilde w_k
&=\min\!\left\{w_{\max},
\left(\frac{\bar N}{\widehat N_k}\right)^\gamma\right\},\\
w_k
&=\frac{\tilde w_k}
{\frac{1}{K}\sum_{j=1}^{K}\tilde w_j+\epsilon}.
\end{aligned}
\label{eq:weights}
\end{equation}
We use hyperparameter values $\gamma=1.0$, $w_{\max}=6.0$, and $\epsilon=10^{-6}$.
These form a centroid-weight vector 
$
\mathbf{w}\ \triangleq\ (w_1,\dots,w_K)^\top \in \mathbb{R}_{>0}^{K}. 
$
We use $w_{a_b}$ to index the weight of the nearest centroid for $\x_b$. The alignment loss is
\begin{equation}
\mathcal{L}^{(i)}_{\text{align}}(\btheta_{E_i};\mathcal{M},\mathbf{w},B_i)=\frac{1}{B}\sum_{b=1}^{B} w_{a_b}\,\big\|E_i(\x_b)-\mathbf{\mu}_{a_b}\big\|_2^2,
\label{eq:align_loss}
\end{equation}
and we update \emph{only the encoder}:
\begin{equation*}
\btheta_{E_i}\leftarrow \btheta_{E_i}-\eta_{\text{align}}\nabla_{\btheta_{E_i}}\mathcal{L}^{(i)}_{\text{align}},
\end{equation*}
treating $(\mathbf{\mu}_k,N_k)$ as constants \emph{(lines 29-36)}.% (no gradient through server statistics).

\subsection{Inference and Thresholding}
For client $i$, the anomaly score is the reconstruction error
\begin{equation}
s(\x)\triangleq\big\|\x-D_i(E_i(\x))\big\|_2^2.
\label{eq:error_score}
\end{equation}
We classify an input as normal if $s(\x)\le\tau_i$ and anomalous otherwise. Since anomaly labels are unavailable during training, we set $\tau_i$ using a \emph{percentile} threshold computed from client $i$'s training errors:
\begin{equation}
\tau_i=\operatorname{Quantile}_{p_i}\!\left(\{s(\x):\x\in\mathcal{D}_i^{\text{train}}\}\right),
\quad
p_i\in[0,1].
\end{equation}
%which controls false positives without labels, is simple and computationally light, and remains stable under gradual encoder shifts induced by periodic alignment.

\subsubsection{Remark.}
The reconstruction, aggregation, and alignment schedule first allows each client to learn a stable local manifold, then injects cross-client structure through global latent centroids, reducing overfitting to site-specific idiosyncrasies. 
Inverse-count weighting in the alignment step further improves coverage of rare or under-represented modes by upweighting centroids with low global support~\cite{anand2010approach, chang2013oversampling, li2020analyzing}.
Because clients upload only a fraction of their latent codes and the server returns a compact set of centroids and counts, per-round communication depends on the uploaded-code count, latent dimension, and number of centroids rather than model size; no parameters or gradients are exchanged~\cite{yamazaki2022deep, jia2024generative}.

We also provide theoretical support for our design in Appendix~\ref{sec: analysis}, establishing (1) stable convergence to a minimizer in the convex case and (2) finite-time convergence to a critical point in the non-convex case.
% Aligning \emph{after} local reconstruction injects global structure without fighting an unstable encoder; nearest-prototype attraction shares information across sites, while inverse-count weights preferentially pull toward underrepresented modes, improving coverage of rare/previously unseen patterns. 
%Mean-one normalization keeps the alignment term comparably scaled as $K$ varies, simplifying learning-rate tuning.

\section{Adversarial Attack Scenarios}
\label{sec:attacks}
\begin{figure*}[t]
    \centering
    \includegraphics[width=0.8\linewidth]{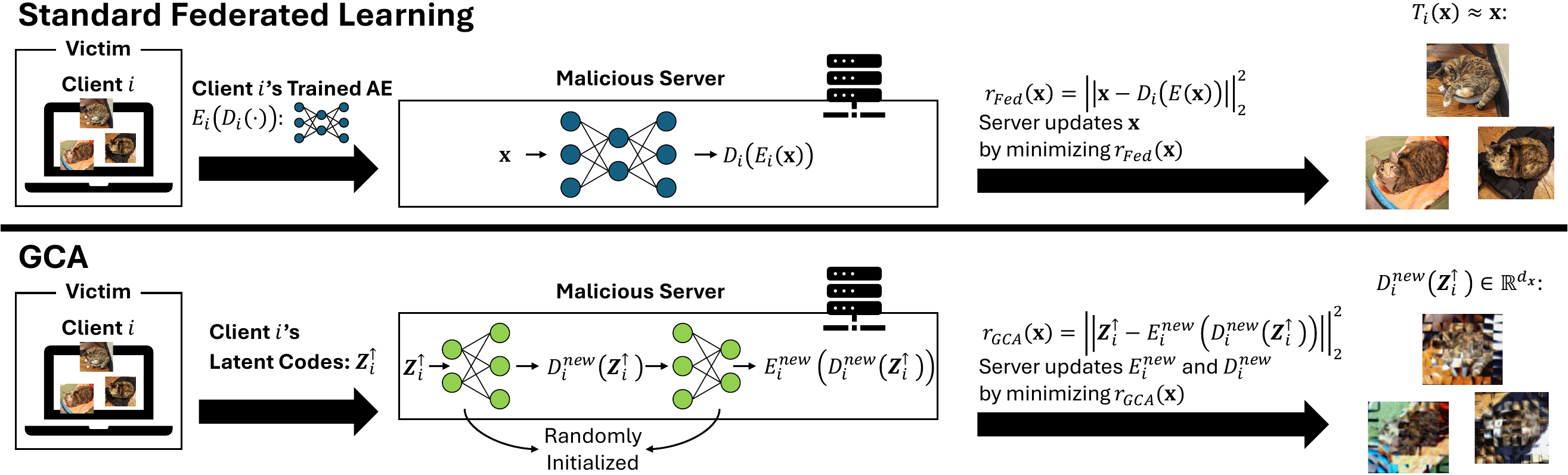}
    \caption{\textbf{Server-side training-data extraction threat models.} \textbf{Top: Standard FL.} The server obtains the victim reconstruction map and performs DeepDream-style input optimization. \textbf{Bottom: GCA.} The server receives uploaded latent codes, trains a fresh surrogate reconstruction map, and applies the same input optimization. Both attacks output the model reconstruction of the optimized input.}
    \label{fig:attack_scens}
\end{figure*}

This section formalizes the server-side threat models considered for \emph{autoencoder-based} federated anomaly detection and our GCA protocol.
We focus on \emph{training data extraction} attacks, since autoencoders are trained to reproduce normal inputs and may therefore leak input training data information through their reconstructions.
We consider a malicious (or honest-but-curious) server under two settings:
(i) \textbf{standard FL}, where the server receives locally trained AE models (parameters or gradients), and
(ii) \textbf{GCA}, where the server receives only uploaded latent codes, and never receives parameters or gradients.
Figure~\ref{fig:attack_scens} depicts the standard-FL inversion and GCA latent-only workflows.
In both cases, the adversary's goal is to reconstruct or closely approximate the private training records of a victim client.

\subsection{Standard FL: DeepDream-style Model Inversion}
\label{subsec:attack_stdfl_deepdream}
DeepDream-style attacks synthesize an input by \emph{directly optimizing the input} against a trained network, following~\cite{mordvintsev2015inceptionism, ghoshal2025reverse}. First, we consider a generic classifier with logits $s(\x)\in\mathbb{R}^{K}$ and a target class $y$.
A canonical DeepDream/feature-visualization objective performs gradient ascent on the class score:
\begin{equation}
\max_{\x\in\mathcal{X}} \; s_y(\x)
\quad\Longrightarrow\quad
\x^{(t+1)}=\Pi_{\mathcal{X}}\!\Big(\x^{(t)}+\eta\,\nabla_{\mathbf{x}} s_y(\x^{(t)})\Big),
\label{eq:deepdream_cls}
\end{equation}
where $\eta>0$ is a step size and $\Pi_{\mathcal{X}}$ projects onto the valid input domain $\mathcal{X}$ (e.g., box constraints such as pixel bounds).

The same input-optimization template applies to autoencoders once we replace the classification score with a reconstruction-based objective.
For a victim client $i$, let the autoencoder define an encoder $E_i:\mathbb{R}^{d_x}\!\to\!\mathbb{R}^{d_z}$ and decoder $D_i:\mathbb{R}^{d_z}\!\to\!\mathbb{R}^{d_x}$, with reconstruction map
\begin{equation}
T_i(\x)\ \triangleq\ D_i(E_i(\x)).
\end{equation}
A natural ``DeepDream-style'' objective for an AE is to find inputs that the AE reconstructs \emph{especially well}, i.e., inputs with small reconstruction error $s_i(\x)$ (the same anomaly score used by client $i$ in GCA's inference stage):
\begin{equation}
\min_{\x\in\mathcal{X}} \; s_i(\x),
\qquad
s_i(\x)\ \triangleq\ \|\x - T_i(\x)\|_2^2.
\label{eq:attack_recon_energy}
\end{equation}
This mirrors \eqref{eq:deepdream_cls} but uses the reconstruction energy instead of a class logit.

In standard FL, the server receives client-side AE models (parameters or gradients). Thus, for a targeted victim client $i$, the server can obtain a \emph{white-box} copy of $(E_i,D_i)$ and compute $\nabla_\mathbf{x} s_i(\x)$ by backpropagation.
Starting from an initialization $\x^{(0)}$ (e.g., random noise or a public sample), the malicious server performs gradient descent on \eqref{eq:attack_recon_energy}:
\begin{equation}
\x^{(t+1)} \;=\; \Pi_{\mathcal{X}}\!\Big(\x^{(t)} - \eta \nabla_{\mathbf{x}} s_i\big(\x^{(t)}\big)\Big),
\quad t=0,1,\dots,T-1,
\label{eq:attack_deepdream_update}
\end{equation}
and outputs $\tilde{\mathbf{x}}=T_i(\x^{(T)})$ as an extracted training-like record.

\paragraph{\textbf{Remark.}}
Autoencoders can be particularly vulnerable to this style of inversion because their training objective explicitly encourages \emph{input reproduction} rather than correct classification.
We provide a theoretical discussion in Appendix~\ref{app:ae_vulnerability}. % (update label)

\subsection{GCA: Latent Code Leakage via Decoder--Encoder ``Flipping''}
\label{subsec:attack_gpa_flip}
In GCA, the server never observes the client's AE parameters or gradients; instead, the server receives only the latent codes.
For a victim client $i$, let
\begin{equation}
Z_i^{\uparrow}=\{\z_n^{(i)}\}_{n=1}^{M_i}\subset\mathbb{R}^{d_z},
\qquad \z_n^{(i)} = E_i(\x_n^{(i)}),
\end{equation}
denote the set of uploaded latents (with $M_i=\rho_i N_i$).
Although the server does not have access to the victim decoder $D_i$, it may still attempt to recover training-like inputs by \emph{learning a surrogate inverse map} from latent space back to input space using only the received latent samples.

\paragraph{\textbf{Decoder--encoder flipping attack.}}
The server constructs a \emph{surrogate} autoencoder whose \emph{input} is a latent code $\z\in\mathbb{R}^{d_z}$ and whose \emph{output} lies in the original input space $\mathbb{R}^{d_x}$.
Concretely, the server defines ``flipped'' encoder and decoder
$
\tilde{D}_i:\mathbb{R}^{d_z}\to\mathbb{R}^{d_x},
$ and $
\tilde{E}_i:\mathbb{R}^{d_x}\to\mathbb{R}^{d_z},
$
and trains the pair $(\tilde{D}_i,\tilde{E}_i)$ \emph{only on client $i$'s uploaded codes} (i.e., without mixing latents across clients) by minimizing a latent-space reconstruction objective
\begin{equation}
\begin{aligned}
\min_{\theta_{\tilde{D}_i},\theta_{\tilde{E}_i}}
\frac{1}{M_i}\sum_{n=1}^{M_i}
\|\z_n^{(i)}- \hat{\mathbf{z}}_n^{(i)}\|_2^2,
\;\;
\hat{\mathbf{z}}_n^{(i)} \triangleq \tilde{E}_i\!\big(\tilde{D}_i(\z_n^{(i)})\big).
\label{eq:flip_obj}
\end{aligned}
\end{equation}
This procedure is supervised in the sense that each latent $\z_n^{(i)}$ serves as its own training target, even though no raw inputs $\x_n^{(i)}$ are available.

\paragraph{\textbf{Surrogate input optimization}.}
After latent-only training, the server defines the surrogate reconstruction map
\begin{equation}
\tilde T_i(\x)\ \triangleq\ \tilde D_i(\tilde E_i(\x))
\label{eq:flip_recover_x}
\end{equation}
and applies the same update in \eqref{eq:attack_deepdream_update}, replacing $T_i$ with $\tilde T_i$. The GCA attack outputs $\tilde T_i(\x^{(T)})$; direct decodes $\tilde D_i(\z)$ are retained only as a GCA-specific source diagnostic.

\paragraph{\textbf{Remark.}}
This attack leverages the fact that GCA reveals a set of latent codes $Z_i^\uparrow$ for each client.
Unlike standard-FL inversion, it does not require backpropagation through the \emph{victim} autoencoder. Because the server trains $(\tilde{D}_i,\tilde{E}_i)$ from scratch using only latent samples, neither $\tilde D_i(\z)$ nor $\tilde T_i(\x)$ is guaranteed to be semantically faithful in input space. Appendix~\ref{app:flip_can_work_but_harder} discusses this underdetermination.

\section{Experimental Evaluation}
\label{sec:experiments}
% \subsection{Main Findings}
% $\bullet$
Through our experimental validation, we find that GCA consistently outperforms isolated Single-Client training and achieves test accuracy comparable to parameter-sharing FL methods. % \\
%$\bullet$
For privacy, GCA provides stronger client-data extraction defense than FedAvg, FedProx, and FedNova, and more consistent defense than DP-FedAvg under the attacks discussed in \textit{Adversarial Attack Scenarios}.% \\
% $\bullet$
We also show that GCA substantially reduces communication costs.

\subsection{Anomaly Detection Evaluation}
\label{subsec:perf}

\paragraph{\textbf{Setup.}}
We evaluate GCA against \emph{Single-Client} (no collaboration), \emph{Centralized} (pooled data), \emph{FedAvg} (sample-weighted model averaging), \emph{FedProx} (FedAvg with proximal regularization), \emph{FedNova} (normalized updates), and \emph{DP-FedAvg} (differential-privacy-style clipped/noisy FedAvg) on five tabular and two vision datasets using three seeds. Each independent-and-identically-distributed (IID) case partitions training data into $C$ disjoint uniform shards. Tabular models use a feed-forward AE (encoder $d_x\to256\to64\to16$, symmetric decoder; ReLU; linear head). Vision models use a convolutional AE with flattened latent dimensions 1,024 (CIFAR-10) and 784 (Fashion-MNIST).

All methods are optimized with Adam (learning rate $10^{-3}$) and trained for $R=100$ global rounds with a mini-batch size of 50.
In GCA, each round consists of $T_{\text{recon}}=5$ local epochs of reconstruction followed by $T_{\text{align}}=5$ local epochs of alignment (Algorithm~\ref{alg:gpa}).
To match total local computation per round, \textbf{FedAvg}, \textbf{FedProx}, \textbf{FedNova}, \textbf{DP-FedAvg}, \textbf{Single-Client}, and \textbf{Centralized} use 10 local epochs per round/epoch-equivalent.
FedAvg and FedProx aggregate full model parameters; FedProx uses a configured proximal coefficient $\mu=10^{-2}$ and an effective implemented proximal-loss multiplier of $10^{-4}$.

In our experiments, all FL variants and GCA use \emph{full client participation} in every round, i.e., $\mathcal{S}_r=\{1,\dots,C\}$ and $\phi=1$.
In GCA, after the reconstruction phase, each client uploads a subset of latent codes with size $|Z_i^\uparrow|=\rho_i N_i$. The server pools the uploaded latent codes as $Z^{(r)}=\bigcup_{i=1}^{C} Z_i^\uparrow$, fits a $K$-component GMM using a modality-specific covariance type, and broadcasts the non-empty component means and effective support counts $(\mu_k,N_k)$ to all clients.
Each client then performs encoder-only alignment using inverse-count weights normalized to have approximately mean one, as described in the \textit{Methodology} section.

Dataset-specific (client count $C$, selected GMM component count $K^\star$, latent upload fraction $\rho$) values are:
\emph{Predict Students' Dropout and Academic Success} $(10,10,0.5)$;
\emph{Adult Income} $(20,10,0.1)$;
\emph{MAGIC} $(10,80,0.25)$;
\emph{Credit Card Fraud Detection} $(50,20,0.1)$;
\emph{Bank Marketing (Full)} $(20,10,0.1)$; 
\emph{CIFAR-10} $(20,20,0.1)$;
\emph{Fashion-MNIST} $(20,10,0.1)$.

% \emph{We use the same number of clients $C$ and the same client-participation schedule for all decentralized baselines.}

To fit the GMM model to the pooled latent codes, we initialize the component means using Torch K-means++ clustering. For tabular datasets, we use a Torch full-covariance GMM with covariance regularization $10^{-6}$ and one initialization. For vision datasets, we use a Torch diagonal-covariance GMM with covariance regularization $0.1$ and 100 initializations. In all cases, we run EM for at most 200 iterations with tolerance $10^{-3}$.

Training uses only ``normal'' data on each client (one-class setting).
For \emph{Adult}, \emph{CIFAR-10}, and \emph{Fashion-MNIST}, we evaluate on the provided test splits.
For the remaining tabular datasets, we form a balanced test set from the anomalous rows and an equal-sized held-out tail of the normal rows.
For tabular preprocessing, we remove highly correlated features (Pearson correlation $>0.6$) and apply standard normalization (fit on training normals and applied to test).
Appendix~\ref{app:implementation-details} documents common settings, preprocessing and partitions, GCA and every baseline, attack implementations, and reproducibility controls.

\paragraph{\textbf{Evaluation Protocol.}}
The baselines are evaluated once per round using their aggregated global models. Since GCA does not produce a global model, it is evaluated by averaging the local models' test accuracies during the reconstruction phase for a certain epoch.
For each seed, we report the highest test accuracy across rounds. We select $K^\star$ as the number of clusters that yields the highest three-seed mean accuracy among all tested $K$ values.

\paragraph{\textbf{Results.}}
\begin{figure}[t]
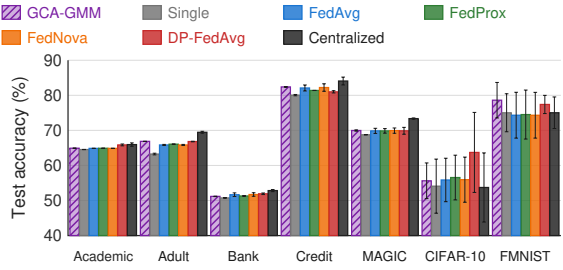

    \centering
    \iidresultsfigure
    \caption{\textbf{IID anomaly-detection test accuracy; higher is better.}
    GCA outperforms Single-Client training on all seven datasets and achieves accuracy comparable to the parameter-sharing FL baselines. Bars show three-seed means and standard deviations. Full results are in Appendix~\ref{app:iid-baseline-results} (Table~\ref{tab:iid-expanded-baselines}).}
    \label{fig:test-acc}
\end{figure}

Figure~\ref{fig:test-acc} reports the primary IID comparison. GCA outperforms Single-Client on all seven datasets and FedAvg on five of seven, with relative changes against FedAvg ranging from $-0.92\%$ to $+5.76\%$. It remains competitive with FedProx, FedNova, and DP-FedAvg, with win/loss counts of 5/2, 5/2, and 4/3, respectively. Exact per-dataset results are provided in Appendix~\ref{app:iid-baseline-results} (Table~\ref{tab:iid-expanded-baselines}).

\paragraph{\textbf{Ablations.}}
Appendix~\ref{app:additional-experiments} supports $K^\star$ with the GMM sweep (Tables~\ref{tab:selected-gca-k}--\ref{tab:iid-gmm-k-sweep}), compares inverse and uniform alignment weights (Table~\ref{tab:iid-gmm-alignment-weighting}), tests three clustering backends (Tables~\ref{tab:iid-kmeans-k-sweep}--\ref{tab:iid-minibatch-kmeans-k-sweep}), and reports the non-IID $K$ sweep and baseline comparison (Tables~\ref{tab:noniid-gmm-k-sweep} and~\ref{tab:noniid-results}).

\begin{figure}[t]
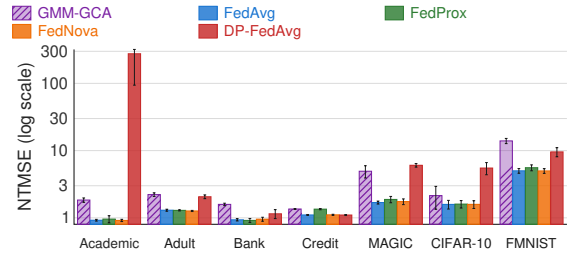

    \centering
    \leakageNTMSEFigure
    \caption{\textbf{Normalized target MSE (NTMSE) under the server-side extraction attack; higher is better for extraction defense.}
    Higher NTMSE means that the attack outputs have larger reconstruction error relative to held-out normal records and are therefore less similar to the client's training data. Bars show three-seed means and standard deviations. Full results are in Appendix~\ref{app:attack-results} (Table~\ref{tab:expanded-attack-primary}).}
    \label{fig:test-attack-ntmse}
\end{figure}
\begin{figure}[t]
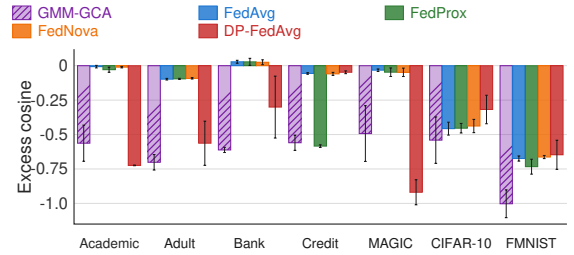

    \centering
    \leakageCosineFigure
    \caption{\textbf{Excess cosine similarity under the server-side extraction attack; lower is better for extraction defense.}
    Lower excess cosine similarity means that the attack outputs are less similar to the client's training data relative to held-out normal records. Bars show three-seed means and standard deviations. Full results are in Appendix~\ref{app:attack-results} (Table~\ref{tab:expanded-attack-primary}).}
    \label{fig:test-attack-cosine}
\end{figure}

\subsection{Adversarial Evaluation}
\label{subsec:adv_eval}
\paragraph{\textbf{Setup}.}
% We evaluate one client per dataset and seed. FedAvg, FedProx, FedNova, and DP-FedAvg use the client's final-round autoencoder.
% For GCA, we also use the client's final-round autoencoder but perform five additional local reconstruction epochs before uploading the configured latent sample, following GCA's standard reconstruction-and-upload procedure. 
% An isolated process then trains a fresh surrogate reconstruction map using only the latent-cycle loss in~\eqref{eq:flip_obj}, without access to client parameters or raw records (Appendix~\ref{app:attack-implementations}).
We implement the two server-side extraction attacks defined in \textit{Adversarial Attack Scenarios}. For FedAvg, FedProx, FedNova, and DP-FedAvg, the server applies the white-box model-inversion attack to the client's final-round autoencoder.
For GCA, the client starts from its final-round checkpoint and performs five additional local reconstruction epochs before uploading the configured latent sample, following GCA's standard reconstruction-and-upload procedure.
An isolated server-side process then trains a fresh decoder--encoder flipping surrogate using only the latent-cycle loss in~\eqref{eq:flip_obj} and applies the same input-optimization update in~\eqref{eq:attack_deepdream_update}, without access to client parameters or raw records (Appendix~\ref{app:attack-implementations}).
% This construction gives every method the strongest model artifact available to the server under its stated communication interface while keeping the client's complete training shard available only to the evaluator.

\paragraph{\textbf{Symmetric attack protocol.}}
Both attacks use the same 100 starts and DeepDream optimizer: independent Adam ($10^{-2}$) runs of at most 1000 steps, with per-start best-state tracking and 10-step relative-improvement stopping at $10^{-3}$. Vision is projected to $[-1,1]$ and standardized tabular inputs are unbounded. The output is the reconstruction $T(\x^\star)$ rather than the optimized input itself, using the client's reconstruction map for the parameter-sharing methods and the latent-only surrogate for GCA. Matching starts and optimization settings isolate the effect of the server-visible reconstruction map rather than giving either communication protocol a different search budget.

\paragraph{\textbf{Metrics.}}
For every training target $\x$, let $r^\star(\x)$ be its MSE (mean-squared error)-nearest output among the 100 reconstructions. Cosine similarity is evaluated against this same $r^\star(\x)$, so the attack cannot use one output for the distance result and another for the similarity result.
With a reference bank of 100 normal records excluded from training, NTMSE is the normalized target MSE, defined as the mean target-to-output MSE divided by the mean target-to-held-out-record MSE. Excess cosine similarity, $\Delta\cos$, is the mean target-to-output cosine similarity minus the mean target-to-held-out-record cosine similarity.
\emph{Higher NTMSE and lower $\Delta\cos$ therefore mean that the attack outputs are less similar to the training shard relative to ordinary held-out records.} The normalization is performed within each dataset, client, and seed; we do not average raw distances across datasets with different feature scales. Complete definitions and all per-dataset values are reported in Appendix~\ref{app:attack-results} and Table~\ref{tab:expanded-attack-primary}.

\paragraph{\textbf{Results.}}
Figures~\ref{fig:test-attack-ntmse} and~\ref{fig:test-attack-cosine} summarize the two leakage measures. We first compare GCA with DP-FedAvg. GCA has higher NTMSE on four of seven datasets and lower excess cosine similarity on five of seven, favoring GCA in 9 of the 14 dataset--metric comparisons. GCA is favored by both metrics on Adult, Bank, Credit, and Fashion-MNIST, whereas DP-FedAvg has higher NTMSE on Academic, MAGIC, and CIFAR-10. More importantly, \textit{DP-FedAvg does not consistently improve over the standard FL variants}: in 4 of the 14 comparisons---both metrics on Credit and excess cosine similarity on CIFAR-10 and Fashion-MNIST---it performs worse than FedAvg, FedProx, and FedNova. In contrast, GCA consistently maintains low target resemblance across these settings. These results indicate that GCA provides more consistent empirical extraction defense than the tested DP-FedAvg configuration.

Compared with the standard parameter-sharing FL methods, GCA has higher NTMSE than FedAvg, FedProx, and FedNova on every dataset, giving 21 of 21 favorable comparisons. It also has lower excess cosine similarity in 20 of 21 comparisons: seven of seven against FedAvg, seven of seven against FedNova, and six of seven against FedProx. The only exception is Credit against FedProx, where GCA has slightly higher NTMSE ($1.352$ versus $1.346$) but slightly higher excess cosine similarity ($-0.559$ versus $-0.584$).

The size of the separation varies by dataset. On Academic, GCA reaches an NTMSE of $1.838$ and excess cosine similarity of $-0.563$, compared with $0.921$ and $-0.007$ for FedAvg. On Adult, the corresponding values are $2.216$ and $-0.702$ for GCA versus $1.298$ and $-0.098$ for FedAvg. GCA achieves its highest NTMSE on Fashion-MNIST ($13.940$), where it also obtains its lowest excess cosine similarity ($-1.003$). These results show that the GCA surrogate outputs remain farther from and less directionally similar to the clients' training records than the outputs obtained from the parameter-sharing FL methods.

Overall, GCA provides stronger and more consistent empirical extraction defense than FedAvg, FedProx, FedNova, and the tested DP-FedAvg configuration. Moreover, GCA reduces per-round communication by $84.14\%$--$99.15\%$, providing a more favorable empirical accuracy--extraction-defense--communication trade-off in our evaluated settings (see the discussion below and Appendix~\ref{app:efficiency-results}). Full results are provided in Appendix~\ref{app:attack-results} (Table~\ref{tab:expanded-attack-primary}).

\subsection{Communication Efficiency}
%In this section, we discuss the communication efficiency of GRA compared to FedAvg-style AE with GRA.
%, accounting for local training, aggregation, and redistribution.
\paragraph{\textbf{Metrics.}}
Let $C$ be the number of clients, $P$ the AE parameter count, $D_z$ the flattened latent dimension, and $U=\sum_{i=1}^{C}L_i$ the total number of uploaded latent codes in one round. All clients participate; every scalar uses 32-bit floating point (FP32), so $\beta=4$ bytes. GCA uploads latent codes and broadcasts $K$ centroids and their $K$ support counts to every client. The round-level payloads are
\begin{equation*}
\begin{aligned}
B_{\mathrm{GCA}}^\uparrow &= \beta U D_z,&
B_{\mathrm{GCA}}^\downarrow &= \beta C K(D_z+1),\\
B_{\mathrm{FL}}^\uparrow &= \beta C P,&
B_{\mathrm{FL}}^\downarrow &= \beta C P.
\end{aligned}
\end{equation*}
Thus $B_{\mathrm{GCA}}=\beta[UD_z+CK(D_z+1)]$ and $B_{\mathrm{FL}}=2\beta CP$ per round. In terms of scaling, FL is $\mathcal{O}(CP)$, whereas GCA is $\mathcal{O}(UD_z+CKD_z)$ and does not depend on model size. Under this accounting, GCA communicates less whenever
$
U < \frac{C\left[2P-K(D_z+1)\right]}{D_z}.
$
The upload fraction controls $U$, while the bottleneck architecture controls $D_z$.

\paragraph{\textbf{Results.}}
Across these evaluated settings, GCA reduces per-round communication by $84.14\%$--$99.15\%$ relative to full-model FL. Appendix~\ref{app:efficiency-protocol} defines the full scope of these settings; Appendix~\ref{app:efficiency-results} lists $C,U,D_z,P$ in Table~\ref{tab:communication-inputs} and reports every dataset and $K\in\{10,20,40,80\}$ in Table~\ref{tab:communication-per-round}.

\section{Conclusion}\label{sec:conclusion}
GCA replaces AE parameter and gradient exchange with sampled latent codes, global centroids, and support counts. Across seven datasets, GCA outperforms FedAvg on five datasets by up to $5.76\%$ while reducing per-round communication by $84.14\%$--$99.15\%$.

Under the matched extraction attack, GCA has higher normalized target MSE in all 21 comparisons and lower excess cosine similarity in 20 of 21 comparisons with FedAvg, FedProx, and FedNova. GCA also provides more consistent extraction defense than the tested DP-FedAvg configuration, outperforming it in 9 of 14 dataset--metric comparisons.

Future work will consider stronger adversaries, dynamic centroids, uncertainty-aware responsibilities, and robust or privacy-preserving clustering mechanisms with formal privacy guarantees.

\clearpage
\newpage
\bibliography{aaai2027}

@inproceedings{mcmahan2017communication,
  title={Communication-efficient learning of deep networks from decentralized data},
  author={McMahan, Brendan and Moore, Eider and Ramage, Daniel and Hampson, Seth and y Arcas, Blaise Aguera},
  booktitle={International Conference on Artificial Intelligence and Statistics},
  pages={1273--1282},
  year={2017},
  organization={PMLR}
}

@article{li2020federated,
  title={Federated optimization in heterogeneous networks},
  author={Li, Tian and Sahu, Anit Kumar and Zaheer, Manzil and Sanjabi, Maziar and Talwalkar, Ameet and Smith, Virginia},
  journal={Proceedings of Machine Learning and Systems},
  volume={2},
  pages={429--450},
  year={2020}
}

@inproceedings{park2024federated,
  title={Federated learning with flexible architectures},
  author={Park, Jong-Ik and Joe-Wong, Carlee},
  booktitle={Joint European Conference on Machine Learning and Knowledge Discovery in Databases},
  pages={143--161},
  year={2024},
  organization={Springer}
}

@article{li2021survey,
  title={A survey on federated learning systems: Vision, hype and reality for data privacy and protection},
  author={Li, Qinbin and Wen, Zeyi and Wu, Zhaomin and Hu, Sixu and Wang, Naibo and Li, Yuan and Liu, Xu and He, Bingsheng},
  journal={IEEE Transactions on Knowledge and Data Engineering},
  volume={35},
  number={4},
  pages={3347--3366},
  year={2021},
  publisher={IEEE}
}

@incollection{thapa2021advancements,
  title={Advancements of federated learning towards privacy preservation: from federated learning to split learning},
  author={Thapa, Chandra and Chamikara, Mahawaga Arachchige Pathum and Camtepe, Seyit A},
  booktitle={Federated Learning Systems: Towards Next-Generation AI},
  pages={79--109},
  year={2021},
  publisher={Springer}
}

@article{zhang2021survey,
  title={A survey on federated learning},
  author={Zhang, Chen and Xie, Yu and Bai, Hang and Yu, Bin and Li, Weihong and Gao, Yuan},
  journal={Knowledge-Based Systems},
  volume={216},
  pages={106775},
  year={2021},
  publisher={Elsevier}
}

@article{li2020review,
  title={A review of applications in federated learning},
  author={Li, Li and Fan, Yuxi and Tse, Mike and Lin, Kuo-Yi},
  journal={Computers \& Industrial Engineering},
  volume={149},
  pages={106854},
  year={2020},
  publisher={Elsevier}
}

@article{konecny2016federated,
  title={Federated learning: Strategies for improving communication efficiency},
  author={Kone{\v{c}}n{\`y}, Jakub and McMahan, H Brendan and Yu, Felix X and Richt{\'a}rik, Peter and Suresh, Ananda Theertha and Bacon, Dave},
  journal={arXiv preprint arXiv:1610.05492},
  year={2016}
}

@inproceedings{ma2025fedsum,
  title={FedSum: Data-Efficient Federated Learning Under Data Scarcity Scenario for Text Summarization},
  author={Ma, Zhiyong and Li, Zhengping and Shi, Yuanjie and Chen, Jian},
  booktitle={Proceedings of the AAAI Conference on Artificial Intelligence},
  volume={39},
  pages={19340--19348},
  year={2025}
}

@article{rieke2020future,
  title={The future of digital health with federated learning},
  author={Rieke, Nicola and Hancox, Jonny and Li, Wenqi and Milletari, Fausto and Roth, Holger R and Albarqouni, Shadi and Bakas, Spyridon and Galtier, Mathieu N and Landman, Bennett A and Maier-Hein, Klaus and others},
  journal={NPJ Digital Medicine},
  volume={3},
  number={1},
  year={2020},
  publisher={Nature Publishing Group UK London}
}

@incollection{long2020federated,
  title={Federated learning for open banking},
  author={Long, Guodong and Tan, Yue and Jiang, Jing and Zhang, Chengqi},
  booktitle={Federated Learning: Privacy and Incentive},
  pages={240--254},
  year={2020},
  publisher={Springer}
}

@article{khowaja2025selffed,
  title={SelfFed: Self-supervised federated learning for data heterogeneity and label scarcity in medical images},
  author={Khowaja, Sunder Ali and Dev, Kapal and Anwar, Syed Muhammad and Linguraru, Marius George},
  journal={Expert Systems with Applications},
  volume={261},
  pages={125493},
  year={2025},
  publisher={Elsevier}
}

@article{kea2023enhancing,
  title={Enhancing anomaly detection in distributed power systems using autoencoder-based federated learning},
  author={Kea, Kimleang and Han, Youngsun and Kim, Tae-Kyung},
  journal={PLOS One},
  volume={18},
  number={8},
  pages={e0290337},
  year={2023},
  publisher={Public Library of Science San Francisco, CA USA}
}

@article{shrestha2024anomaly,
  title={Anomaly detection based on lstm and autoencoders using federated learning in smart electric grid},
  author={Shrestha, Rakesh and Mohammadi, Mohammadreza and Sinaei, Sima and Salcines, Alberto and Pampliega, David and Clemente, Raul and Sanz, Ana Lourdes and Nowroozi, Ehsan and Lindgren, Anders},
  journal={Journal of Parallel and Distributed Computing},
  volume={193},
  pages={104951},
  year={2024},
  publisher={Elsevier}
}

@article{novoa2023fast,
  title={Fast deep autoencoder for federated learning},
  author={Novoa-Paradela, David and Fontenla-Romero, Oscar and Guijarro-Berdi{\~n}as, Bertha},
  journal={Pattern Recognition},
  volume={143},
  pages={109805},
  year={2023},
  publisher={Elsevier}
}

@inproceedings{vucovich2023anomaly,
  title={Anomaly detection via federated learning},
  author={Vucovich, Marc and Tarcar, Amogh and Rebelo, Penjo and Rahman, Abdul and Nandakumar, Dhruv and Redino, Christopher and Choi, Kevin and Schiller, Robert and Bhattacharya, Sanmitra and Veeramani, Balaji and others},
  booktitle={2023 33rd International Telecommunication Networks and Applications Conference},
  pages={259--266},
  year={2023},
  organization={IEEE}
}

@article{jiang2022privacy,
  title={Privacy-preserving high-dimensional data collection with federated generative autoencoder},
  author={Jiang, Xue and Zhou, Xuebing and Grossklags, Jens},
  journal={Proceedings on Privacy Enhancing Technologies},
  year={2022}
}

@article{chen2023auto,
  title={Auto-encoders in deep learning—a review with new perspectives},
  author={Chen, Shuangshuang and Guo, Wei},
  journal={Mathematics},
  volume={11},
  number={8},
  pages={1777},
  year={2023},
  publisher={MDPI}
}

@article{ghoshal2025reverse,
  title={Reverse-Engineering Autoencoders: Optimized Gradient-Inversion Attacks for Sensitive Data Extraction},
  author={Ghoshal, Arjun and Das, Rajdeep and Kundu, Arunava and De, Indrajit},
  journal={Authorea Preprints},
  year={2025},
  publisher={Authorea}
}

@inproceedings{suganuma2018exploiting,
  title={Exploiting the potential of standard convolutional autoencoders for image restoration by evolutionary search},
  author={Suganuma, Masanori and Ozay, Mete and Okatani, Takayuki},
  booktitle={International Conference on Machine Learning},
  year={2018},
  organization={PMLR}
}

@article{steck2020autoencoders,
  title={Autoencoders that don't overfit towards the identity},
  author={Steck, Harald},
  journal={Advances in Neural Information Processing Systems},
  volume={33},
  pages={19598--19608},
  year={2020}
}

@inproceedings{chen2017outlier,
  title={Outlier detection with autoencoder ensembles},
  author={Chen, Jinghui and Sathe, Saket and Aggarwal, Charu and Turaga, Deepak},
  booktitle={Proceedings of the 2017 SIAM International Conference on Data Mining},
  pages={90--98},
  year={2017},
  organization={SIAM}
}

@article{singh2019detailed,
  title={Detailed comparison of communication efficiency of split learning and federated learning},
  author={Singh, Abhishek and Vepakomma, Praneeth and Gupta, Otkrist and Raskar, Ramesh},
  journal={arXiv preprint arXiv:1909.09145},
  year={2019}
}

@inproceedings{gao2020end,
  title={End-to-End Evaluation of Federated Learning and Split Learning for Internet of Things},
  author={Gao, Yansong and Kim, Minki and Abuadbba, Sharif and Kim, Yeonjae and Thapa, Chandra and Kim, Kyuyeon and Camtep, Seyit A and Kim, Hyoungshick and Nepal, Surya},
  booktitle={2020 International Symposium on Reliable Distributed Systems (SRDS)},
  pages={91--100},
  year={2020},
  organization={IEEE}
}

@inproceedings{jia2024generative,
  title={Generative latent coding for ultra-low bitrate image compression},
  author={Jia, Zhaoyang and Li, Jiahao and Li, Bin and Li, Houqiang and Lu, Yan},
  booktitle={Proceedings of the IEEE/CVF Conference on Computer Vision and Pattern Recognition},
  pages={26088--26098},
  year={2024}
}

@inproceedings{yamazaki2022deep,
  title={Deep feature compression using rate-distortion optimization guided autoencoder},
  author={Yamazaki, Meguru and Kora, Yuichiro and Nakao, Takanori and Lei, Xuying and Yokoo, Kaoru},
  booktitle={IEEE International Conference on Image Processing},
  pages={1216--1220},
  year={2022},
  organization={IEEE}
}

@article{tschannen2018recent,
  title={Recent advances in autoencoder-based representation learning},
  author={Tschannen, Michael and Bachem, Olivier and Lucic, Mario},
  journal={arXiv preprint arXiv:1812.05069},
  year={2018}
}

@article{li2020analyzing,
  title={Analyzing overfitting under class imbalance in neural networks for image segmentation},
  author={Li, Zeju and Kamnitsas, Konstantinos and Glocker, Ben},
  journal={IEEE Transactions on Medical Imaging},
  volume={40},
  number={3},
  pages={1065--1077},
  year={2020},
  publisher={IEEE}
}

@article{chang2013oversampling,
  title={Oversampling to overcome overfitting: exploring the relationship between data set composition, molecular descriptors, and predictive modeling methods},
  author={Chang, Chia-Yun and Hsu, Ming-Tsung and Esposito, Emilio Xavier and Tseng, Yufeng J},
  journal={Journal of Chemical Information and Modeling},
  volume={53},
  number={4},
  pages={958--971},
  year={2013},
  publisher={ACS Publications}
}

@article{anand2010approach,
  title={An approach for classification of highly imbalanced data using weighting and undersampling},
  author={Anand, Ashish and Pugalenthi, Ganesan and Fogel, Gary B and Suganthan, PN},
  journal={Amino Acids},
  volume={39},
  number={5},
  pages={1385--1391},
  year={2010},
  publisher={Springer}
}

@inproceedings{mohammed2020machine,
  title={Machine learning with oversampling and undersampling techniques: overview study and experimental results},
  author={Mohammed, Roweida and Rawashdeh, Jumanah and Abdullah, Malak},
  booktitle={2020 11th International Conference on Information and Communication Systems (ICICS)},
  pages={243--248},
  year={2020},
  organization={IEEE}
}

@article{cheng2021improved,
  title={Improved autoencoder for unsupervised anomaly detection},
  author={Cheng, Zhen and Wang, Siwei and Zhang, Pei and Wang, Siqi and Liu, Xinwang and Zhu, En},
  journal={International Journal of Intelligent Systems},
  volume={36},
  number={12},
  pages={7103--7125},
  year={2021},
  publisher={Wiley Online Library}
}

@article{ibidunmoye2017adaptive,
  title={Adaptive anomaly detection in performance metric streams},
  author={Ibidunmoye, Olumuyiwa and Rezaie, Ali-Reza and Elmroth, Erik},
  journal={IEEE Transactions on Network and Service Management},
  volume={15},
  number={1},
  pages={217--231},
  year={2017},
  publisher={IEEE}
}

@article{kiet2025statistical,
  title={Statistical Inference for Autoencoder-based Anomaly Detection after Representation Learning-based Domain Adaptation},
  author={Kiet, Tran Tuan and Loi, Nguyen Thang and Duy, Vo Nguyen Le},
  journal={arXiv preprint arXiv:2508.07049},
  year={2025}
}

@article{van2021anomaly,
  title={An anomaly detection approach to identify chronic brain infarcts on MRI},
  author={Van Hespen, Kees M and Zwanenburg, Jaco JM and Dankbaar, Jan W and Geerlings, Mirjam I and Hendrikse, Jeroen and Kuijf, Hugo J},
  journal={Scientific Reports},
  volume={11},
  number={1},
  pages={7714},
  year={2021},
  publisher={Nature Publishing Group UK London}
}

@article{mclachlan2014number,
  title={On the number of components in a Gaussian mixture model},
  author={McLachlan, Geoffrey J and Rathnayake, Suren},
  journal={Wiley Interdisciplinary Reviews: Data Mining and Knowledge Discovery},
  volume={4},
  number={5},
  pages={341--355},
  year={2014},
  publisher={Wiley Online Library}
}

@inproceedings{karimi2016linear,
  title={Linear convergence of gradient and proximal-gradient methods under the polyak-{\l}ojasiewicz condition},
  author={Karimi, Hamed and Nutini, Julie and Schmidt, Mark},
  booktitle={Joint European Conference on Machine Learning and Knowledge Discovery in Databases},
  pages={795--811},
  year={2016},
  organization={Springer}
}

@article{bottou2018optimization,
  title={Optimization methods for large-scale machine learning},
  author={Bottou, L{\'e}on and Curtis, Frank E and Nocedal, Jorge},
  journal={SIAM Review},
  volume={60},
  number={2},
  pages={223--311},
  year={2018},
  publisher={SIAM}
}

@article{tseng2001convergence,
  title={Convergence of a block coordinate descent method for nondifferentiable minimization},
  author={Tseng, Paul},
  journal={Journal of Optimization Theory and Applications},
  volume={109},
  number={3},
  pages={475--494},
  year={2001},
  publisher={Springer}
}

@article{bolte2014proximal,
  title={Proximal alternating linearized minimization for nonconvex and nonsmooth problems},
  author={Bolte, J{\'e}r{\^o}me and Sabach, Shoham and Teboulle, Marc},
  journal={Mathematical Programming},
  volume={146},
  number={1},
  pages={459--494},
  year={2014},
  publisher={Springer}
}

@article{mordvintsev2015inceptionism,
  title={Inceptionism: Going deeper into neural networks},
  author={Mordvintsev, Alexander and Olah, Christopher and Tyka, Mike},
  journal={Google research blog},
  volume={20},
  number={14},
  pages={5},
  year={2015}
}

@article{Frankel_2014,
   title={Splitting Methods with Variable Metric for Kurdyka–Łojasiewicz Functions and General Convergence Rates},
   volume={165},
   ISSN={1573-2878},
   url={http://dx.doi.org/10.1007/s10957-014-0642-3},
   DOI={10.1007/s10957-014-0642-3},
   number={3},
   journal={Journal of Optimization Theory and Applications},
   publisher={Springer Science and Business Media LLC},
   author={Frankel, Pierre and Garrigos, Guillaume and Peypouquet, Juan},
   year={2014},
   month=sep, pages={874–900}
   }

@article{zhou2021vae,
  title={VAE-based deep SVDD for anomaly detection},
  author={Zhou, Yu and Liang, Xiaomin and Zhang, Wei and Zhang, Linrang and Song, Xing},
  journal={Neurocomputing},
  volume={453},
  pages={131--140},
  year={2021},
  publisher={Elsevier}
}

@inproceedings{zong2018deep,
  title={Deep autoencoding gaussian mixture model for unsupervised anomaly detection},
  author={Zong, Bo and Song, Qi and Min, Martin Renqiang and Cheng, Wei and Lumezanu, Cristian and Cho, Daeki and Chen, Haifeng},
  booktitle={International conference on learning representations},
  year={2018}
}

@inproceedings{abadi2016deep,
  author    = {Abadi, Mart{\'i}n and Chu, Andy and Goodfellow, Ian
               and McMahan, H. Brendan and Mironov, Ilya and Talwar, Kunal
               and Zhang, Li},
  title     = {Deep Learning with Differential Privacy},
  booktitle = {Proceedings of the 2016 ACM SIGSAC Conference on
               Computer and Communications Security},
  pages     = {308--318},
  publisher = {ACM},
  year      = {2016},
  doi       = {10.1145/2976749.2978318}
}

@inproceedings{mcmahan2018learning,
  author    = {McMahan, H. Brendan and Ramage, Daniel and Talwar, Kunal
               and Zhang, Li},
  title     = {Learning Differentially Private Recurrent Language Models},
  booktitle = {International Conference on Learning Representations},
  year      = {2018}
}

@inproceedings{tan2022fedproto,
  author    = {Tan, Yue and Long, Guodong and Liu, Lu and Zhou, Tianyi
               and Lu, Qinghua and Jiang, Jing and Zhang, Chengqi},
  title     = {{FedProto}: Federated Prototype Learning across
               Heterogeneous Clients},
  booktitle = {Proceedings of the AAAI Conference on Artificial Intelligence},
  volume    = {36},
  pages     = {8432--8440},
  year      = {2022},
  doi       = {10.1609/aaai.v36i8.20819}
}

@article{sattler2020communication,
  author  = {Sattler, Felix and Marban, Arturo and Rischke, Roman
             and Samek, Wojciech},
  title   = {Communication-Efficient Federated Distillation},
  journal = {arXiv preprint arXiv:2012.00632},
  year    = {2020}
}

@inproceedings{wang2020tackling,
  author    = {Wang, Jianyu and Liu, Qinghua and Liang, Hao and Joshi, Gauri and Poor, H. Vincent},
  title     = {Tackling the Objective Inconsistency Problem in Heterogeneous Federated Optimization},
  booktitle = {Advances in Neural Information Processing Systems},
  volume    = {33},
  year      = {2020}
}

\onecolumn
\appendix

% Enable appendix section numbering: A, B, C, ...
\setcounter{secnumdepth}{1}

% Restore the normal contents-writing command
\makeatletter
\let\addcontentsline\AAAIOriginalAddContentsLine
\makeatother

\startcontents[appendix]

\printcontents[appendix]{l}{1}{%
  \section*{Appendix Contents}
  \setcounter{tocdepth}{2}
}

\clearpage
\newpage
\section{Implementation Details and Experimental Protocols}
\label{app:implementation-details}

This appendix specifies the shared experimental settings and the method-specific implementations used in the reported evaluations. Data partitions, architectures, optimization settings, training budgets, and participation schedules are matched where applicable.

\subsection{Common Experimental Settings}
\label{app:common-settings}

\paragraph{Architectures and optimization.}
For tabular datasets, we use a feed-forward autoencoder with encoder dimensions $d_x\!\to\!256\!\to\!64\!\to\!16$ and a symmetric decoder, ReLU activations, and a linear output layer. For CIFAR-10 and Fashion-MNIST, we use a three-layer convolutional encoder with 128, 128, and 16 output channels and a symmetric decoder; the flattened latent dimensions are 1,024 and 784, and the models contain 339,219 and 334,609 parameters, respectively. All methods use Adam with learning rate $10^{-3}$, mini-batch size 50, StepLR step size 1,000 and multiplier 0.1, and three random seeds $\{100,200,300\}$.
For both vision datasets, seeds 100, 200, and 300 select normal/anomaly class pairs $2/8$, $0/4$, and $9/5$; CIFAR-10 uses 5,000 normal training samples (250 per client), while Fashion-MNIST uses 6,000 (300 per client).
Thus, for the vision benchmarks, each seed jointly specifies the random state and the normal/anomaly class pair. The variation across the three runs therefore reflects both stochastic variation and variation across these three task instantiations rather than repeated training on one fixed class pair.

\paragraph{Training budget and participation.}
Training uses $R=100$ global rounds. GCA performs $T_{\mathrm{recon}}=5$ reconstruction epochs followed by $T_{\mathrm{align}}=5$ alignment epochs per round. Single-Client, Centralized, FedAvg, FedProx, FedNova, and DP-FedAvg use 10 local epochs per round or the centralized epoch-equivalent. Unless otherwise stated, all decentralized methods use full client participation.

\paragraph{Data preprocessing and anomaly decisions.}
Training uses normal data only. Each client scores a sample using reconstruction error and sets its decision threshold to the $0.75$ quantile of its training reconstruction errors. Tabular preprocessing removes features with Pearson correlation greater than $0.6$ and applies standard normalization fitted on the normal training data.

\subsection{Data Partitioning and Configuration Selection}
\label{app:data-partitioning}

\paragraph{Uniform-shard setting.}
For the primary IID experiments, the normal training data are randomly partitioned into disjoint uniform shards. The same realized partition is reused by every method for a given dataset and random seed.

\paragraph{Non-IID Dirichlet setting.}
For the heterogeneous experiments, client preferences are sampled from a Dirichlet distribution with concentration parameter $\alpha=0.1$. Normal samples are represented by feature descriptors and grouped into $\min(20,C)$ groups using MiniBatchKMeans with 10 initializations, at most 200 iterations, and batch size $\min(2{,}048,N)$; group-specific client preferences are then allocated under long-tail quotas with a target maximum-to-minimum ratio of 10:1. Tabular descriptors are the standardized features, whereas vision descriptors are normalized pixels adaptively pooled to $8\times8$; only the five tabular non-IID results are reported. The same realized partition is reused by every method for a given dataset and seed. Because this experiment changes client data distributions rather than client feature dimensions, we refer to it as non-IID statistical heterogeneity.

\paragraph{Accuracy reporting and GCA component-count selection.}
GCA is evaluated after each reconstruction and alignment epoch, whereas the baselines are evaluated after each round. For each seed, we report the highest reconstruction-epoch test accuracy for GCA and the highest round-level test accuracy for each baseline. For GCA-GMM, $K^\star$ maximizes the three-seed mean over $K\in\{10,20,40,80\}$, with exact ties resolved in favor of the smaller $K$. The same dataset-specific $K^\star$ is used for the main IID comparison and the extraction evaluation, while communication is reported for every tested $K$. The weighting diagnostic reports the highest alignment-epoch accuracy at this $K^\star$.

\subsection{GCA Implementation}
\label{app:gca-implementation}

\paragraph{Local reconstruction and latent upload.}
Each participating client first updates both its encoder and decoder using reconstruction loss. It then uploads a uniformly sampled subset of its normal latent codes, with upload fraction $\rho_i$ given in the main experimental setup. No raw data, model parameters, or gradients are uploaded.

\paragraph{Latent-coordinate compatibility.}
For each dataset and seed, clients use the same encoder architecture, latent dimension, and preprocessing. The experiments therefore assume that the resulting latent coordinates are sufficiently compatible for pooled clustering. Arbitrary client-specific permutations or transformations of the latent space are not considered.

\paragraph{Server clustering and broadcast.}
The server pools the uploaded latent codes and fits one of four clustering backends: GMM, K-means, Soft K-means, or Mini-Batch K-means. For every backend, $K\in\{10,20,40,80\}$. The server broadcasts all $K$ returned centroids and their support counts. For a GMM, $N_k=\sum_n\gamma_{nk}$ is the effective support count; for hard clustering, $N_k$ is the cluster size.

For GMM, component means are initialized with Torch K-means++. Tabular datasets use a Torch full-covariance GMM with regularization $10^{-6}$ and one initialization, while vision datasets use a Torch diagonal-covariance GMM with regularization $0.1$ and 100 initializations. EM runs for at most 200 iterations with tolerance $10^{-3}$. K-means uses 10 initializations; Soft K-means uses temperature 1.0 after K-means initialization with 10 initializations; and Mini-Batch K-means uses batch size 4,096 and 10 initializations.

\paragraph{Inverse-support centroid alignment.}
Each client assigns its local codes to the nearest centroid and scales each per-sample alignment term using an inverse-support weight derived from the assigned centroid's effective count. Let
\[
\begin{aligned}
\widehat N_k&=\max(N_k,\epsilon),
&\bar N&=\frac{1}{K}\sum_{k=1}^{K}\widehat N_k,\\
\widetilde w_k
&=\min\!\left\{w_{\max},
\left(\frac{\bar N}{\widehat N_k}\right)^\gamma\right\},\\
w_k
&=\frac{\widetilde w_k}
{\frac{1}{K}\sum_{j=1}^{K}\widetilde w_j+\epsilon},
&\gamma&=1.0,\quad w_{\max}=6.0,\quad\epsilon=10^{-6}.
\end{aligned}
\]
The normalization gives $K^{-1}\sum_k w_k\approx1$. Alignment updates only the encoder, while the decoder remains local and inference continues to use reconstruction-error thresholding.

\paragraph{Uniform-weight GMM ablation.}
To isolate the contribution of inverse-support weighting, we compare the default GCA-GMM method with a uniform-weight variant in which $w_k=1$ for every nonempty centroid. Both variants use the same $K\in\{10,20,40,80\}$ sweep, data partitions, latent uploads, GMM fitting, nearest-centroid assignments, training schedule, and optimization settings. This ablation is limited to the GMM backend.

\subsection{Baseline Implementations}
\label{app:baseline-implementations}

\paragraph{Single-Client.}
Each client trains an autoencoder only on its local normal data, with no server communication or aggregation. The architecture, optimizer, and total local epoch budget match the decentralized methods. Reported performance is averaged across clients.

\paragraph{Centralized.}
A single autoencoder is trained on the pooled normal training data using the same architecture and optimizer. Its training budget is matched to the epoch-equivalent used by the decentralized methods. Centralized training is used as a reference and is not a privacy-preserving deployment.

\paragraph{FedAvg.}
Clients train the complete autoencoder locally and upload model parameters or parameter deltas. The server aggregates client updates using local-sample-size weighting and broadcasts the resulting global autoencoder~\cite{mcmahan2017communication}.

\paragraph{FedProx.}
FedProx follows the FedAvg communication protocol and adds a proximal penalty to the complete autoencoder parameters~\cite{li2020federated}.
The local objective is
\[
\mathcal{L}^{(i)}_{\mathrm{FedProx}}
=
\mathcal{L}^{(i)}_{\mathrm{recon}}
+
10^{-4}
\left\|
\theta_i-\theta^{(r)}
\right\|_2^2,
\]
where the effective coefficient $10^{-4}$ results from the configured $\mu=10^{-2}$ under the implemented loss scaling.

\paragraph{FedNova.}
FedNova uses the same local autoencoder, optimizer, participation schedule, and local epoch budget as FedAvg. Each client transmits a normalized local update, and the server applies normalized aggregation following FedNova~\cite{wang2020tackling}. With 10 local epochs, $\tau_i=10n_i^{\mathrm{batch}}$, and the effective normalization is $\tau_{\mathrm{eff}}=\sum_i p_i\tau_i$, where $p_i$ is the sample-size weight of client $i$.

\paragraph{DP-FedAvg.}
DP-FedAvg follows the FedAvg training and communication protocol while clipping each client's full model delta to global $\ell_2$ norm 1.0 at the server. After sample-size-weighted aggregation, the server adds independent Gaussian noise to each aggregate tensor with standard deviation $\sigma/C$, where $\sigma=1.0$ and $C$ is the number of clients. Because the current implementation does not include a complete sampling and composition accountant, this configuration is treated as a noisy-update utility and extraction baseline rather than a certified $(\varepsilon,\delta)$-DP mechanism~\cite{abadi2016deep,mcmahan2018learning}.

\begin{table*}[t]
\centering
\papertablesize
\caption{Implementation summary for all compared approaches.}
\label{tab:all-method-implementations}
\begin{tabularx}{\textwidth}{@{}lXXXl@{}}
\toprule
Method & Client objective & Communicated object & Server operation & Key setting \\
\midrule
Single-Client & Reconstruction & None & None & 10 local epochs \\
Centralized & Reconstruction & Pooled training data & Central training & Epoch-equivalent budget \\
FedAvg & Reconstruction & AE model update & Sample-size-weighted averaging & 10 local epochs \\
FedProx & Reconstruction plus proximal term & AE model update & Sample-size-weighted averaging & Configured $\mu=10^{-2}$; effective coefficient $10^{-4}$ \\
FedNova & Reconstruction & Normalized AE update & Normalized aggregation & $\tau_i=10n_i^{\mathrm{batch}}$; sample-size-weighted normalization \\
DP-FedAvg & Reconstruction & Clipped and noisy AE update & FedAvg aggregation & clip 1.0; $\sigma=1.0$ \\
GCA & Reconstruction plus centroid alignment & Latent codes / centroid statistics & Clustering & $K\in\{10,20,40,80\}$ \\
GCA-GMM (uniform weights) & Reconstruction plus centroid alignment & Latent codes / centroid statistics & GMM clustering & $w_k=1$; $K\in\{10,20,40,80\}$ \\
\bottomrule
\end{tabularx}
\end{table*}

\subsection{Server-Side Extraction Attack Implementations}
\label{app:attack-implementations}

\paragraph{Parameter-sharing baselines.}
For FedAvg, FedProx, FedNova, and DP-FedAvg, the server has white-box access to one victim client's selected round-100 autoencoder and performs no additional local update. Using the same 100 attack initializations across all methods, the server minimizes reconstruction error with Adam at learning rate $10^{-2}$ for at most 1000 iterations. Each initialization uses independent best-state tracking and stops after 10 consecutive steps with relative improvement below $10^{-3}$. Vision inputs are projected to $[-1,1]$, whereas standardized tabular inputs are left unbounded.

\paragraph{GCA latent-only attack.}
Starting from the same client's round-100 checkpoint, GCA performs five reconstruction epochs and uploads the configured latent subset. The server trains an independently initialized surrogate autoencoder using only these latent codes, with Adam at learning rate $10^{-2}$, batch size 50, a 10\% validation split, at most 1000 epochs, minimum relative improvement $10^{-4}$, and patience 20. It then applies the same 100-start input-space optimization, early-stopping rule, and input-domain constraints used for the parameter-sharing methods. The surrogate receives no victim parameters, gradients, or raw records.

\paragraph{Metrics.}
The reported attack output is $T_i(\x^\star)$ for the parameter-sharing methods and $\widetilde T_i(\x^\star)$ for GCA, rather than the optimized input $\x^\star$ itself. For every target in the evaluated client's training shard, we select its MSE-nearest output and report normalized target MSE and excess cosine similarity relative to 100 normal records excluded from training. Cosine similarity is computed using the same MSE-selected output. Means and standard deviations are reported over the same three seed configurations used in the performance evaluation.

\subsection{Deterministic Communication Accounting}
\label{app:efficiency-protocol}

We compare the deterministic per-round tensor payload of GCA-GMM for every tested $K\in\{10,20,40,80\}$ with FedAvg, FedProx, FedNova, and DP-FedAvg. Let $C$ denote the number of clients, $U$ the total uploaded latent codes, $D_z$ the flattened latent dimension, and $P$ the complete autoencoder parameter count. All communicated tensors are FP32. GCA uploads $UD_z$ latent scalars and broadcasts $K(D_z+1)$ centroid-and-count scalars to each client; each FL baseline uploads and broadcasts $P$ model scalars per client. The totals include every client copy but exclude transport headers, serialization overhead, acknowledgements, checkpoint I/O, and optimizer state. This accounting reports per-round tensor payload only and does not measure runtime or total communication to a target accuracy.

\subsection{Computing Environment}
\label{app:computing-environment}
Experiments were executed on NVIDIA H100 and H200 GPUs. The same numerical configurations and software environment were used across accelerator types.

\subsection{Experimental Controls}
\label{app:reproducibility}

For a given dataset and seed configuration, every method uses the same training and test data and the same client partition. Autoencoder capacity, local optimizer, batch size, local epoch budget, and client participation are matched where applicable. GCA-GMM uses the same selected $K^\star$ in the main performance and extraction evaluations, while communication is reported for every tested $K$. All methods use the same three seed configurations; for the vision datasets, these configurations also specify the normal/anomaly class pairs described above.

\clearpage
\newpage
\section{Positioning Relative to Prototype-Based FL and
Differentially Private AE-FL}
\label{app:gca-vs-prototype-fl}

\subsection{Class-Prototype Aggregation versus GCA}
FedProto communicates class-conditioned feature prototypes~\cite{tan2022fedproto}. For each observed class, a client computes a local prototype as the mean of the corresponding sample embeddings.
The server then aggregates prototypes having the same class label, using their local class support, and returns global class prototypes. Each client optimizes a supervised classification loss together with a regularization term that moves its local class prototypes toward the corresponding global ones.

GCA shares only the broad idea of communicating representation-space summaries. Unlike FedProto, it operates on unlabeled normal data, uploads sampled latent codes rather than class-indexed means, and discovers round-specific geometric modes by clustering. Clients align their encoder outputs with the resulting centroids using inverse-support weights; decoders remain local, and inference uses reconstruction error.

\begingroup
\papertablesize
\setlength{\tabcolsep}{4pt}
\renewcommand{\arraystretch}{1.15}
\begin{papertableblock}
\papertablecaption{
Methodological distinction between the original supervised FedProto formulation~\cite{tan2022fedproto} and GCA.
}
\label{tab:gca-vs-prototype-fl}
\begin{tabularx}{\columnwidth}{@{}p{0.18\columnwidth}XX@{}}
\toprule
Dimension
&
FedProto-style class-prototype FL
&
GCA
\\
\midrule

Learning task
&
Supervised classification
&
One-class, unlabeled anomaly detection
\\

Semantic identity
&
Class labels identify corresponding prototypes across clients
&
No labels; latent modes are discovered from the pooled codes
\\

Client uplink
&
One local mean embedding for each observed class
&
Sampled individual latent codes from normal data
\\

Server operation
&
Aggregate local prototypes that share the same class label
&
Fit a clustering model to pooled latent codes
\\

Server downlink
&
Global class prototypes
&
Global latent centroids and effective support counts
\\

Use of support
&
Local class counts weight the aggregation of same-label prototypes
&
Inverse global support scales local alignment through
$w_k\propto 1/N_k$ with mean-one normalization
\\

Client objective
&
Classification loss plus class-prototype regularization
&
AE reconstruction loss plus encoder-only centroid alignment
\\

Model requirements
&
Shared class semantics and a compatible prototype dimension;
local architectures may differ
&
A common latent dimension and sufficiently compatible latent
coordinates; local decoders remain independent
\\

Inference
&
Classifier output or distance to class prototypes
&
Reconstruction error with a training-error quantile threshold
\\

Server-visible information
&
Class-level feature means; no model parameters or gradients in the
standard FedProto protocol
&
Sampled latent codes in the uplink and centroid statistics in the
downlink; no model parameters or gradients
\\

\bottomrule
\end{tabularx}
\end{papertableblock}
\endgroup

The original FedProto formulation requires class identities and is therefore not directly applicable to GCA's unlabeled, one-class setting. GCA avoids parameter and gradient exchange, but its sampled latent codes remain information-bearing and are evaluated under the latent-only extraction threat model. Pooling latent codes further assumes a common latent dimension and sufficiently compatible latent coordinates across clients; centroid communication alone does not guarantee robustness to arbitrary feature-space heterogeneity.

\subsection{Differential Privacy versus Withholding the AE
Reconstruction Map}
Differential privacy and communication architecture address different exposure mechanisms. DP-SGD clips individual contributions, adds calibrated noise, and accounts for privacy loss across optimization steps~\cite{abadi2016deep}. User-level DP-FL similarly clips and perturbs client contributions while composing privacy loss across communication rounds~\cite{mcmahan2018learning}. When the protected unit, adjacency relation, clipping rule, sampling process, noise mechanism, and privacy accountant are fully specified, these methods provide a formal bound on how much the distribution of the released model can change when one protected record or client is modified.

Such a guarantee does not, however, remove the trained AE's reconstruction map from the server-visible interface when the final encoder and decoder are communicated. The AE is still optimized so that $D(E(\mathbf{x}))\approx\mathbf{x}$ on normal training data, and a server possessing the model can directly evaluate $D\circ E$ or optimize candidate inputs toward low reconstruction error. A finite $(\varepsilon,\delta)$ guarantee does not assert that the trained model can never generate a training-like reconstruction. Rather, DP limits the influence of the protected record or client on the released model. It may consequently reduce record-specific memorization and attack success, but stronger noise can also degrade the representation and reduce anomaly-detection utility. White-box reconstruction therefore remains a relevant empirical attack surface for a DP-trained AE whose complete model is visible to the server~\cite{ghoshal2025reverse}. 
A noise multiplier by itself is not a certified privacy guarantee. Computing an $(\varepsilon,\delta)$ budget additionally depends on the clipping granularity, sampling rate, number of rounds, client participation, adjacency definition, and accounting method~\cite{abadi2016deep,mcmahan2018learning}. The complete DP-FedAvg implementation, including clipping, noise, and accounting status, is specified together with all other approaches in Appendix~\ref{app:implementation-details}.

GCA makes a complementary systems choice. It withholds AE parameters and gradients from the server, thereby preventing direct white-box access to the client reconstruction maps, but it exposes sampled latent codes instead. GCA therefore does not eliminate privacy risk and does not provide a formal DP guarantee. Its privacy properties are evaluated under a latent-specific extraction threat model. In future work, a formal DP mechanism could be incorporated by clipping and perturbing the uploaded latent codes, privatizing the released centroid statistics, and accounting for their repeated release across communication rounds.

\clearpage
\newpage
\section{Theoretical Analysis}
\label{sec: analysis}
We analyze one global round of GCA for a client $i$ under the schedule of the \emph{Methodology} section. We set the number of local epochs per phase to one, i.e., $T_{\text{recon}}=T_{\text{align}}=1$. Thus, each round consists of exactly one full-batch gradient-descent step on the \emph{reconstruction} loss \eqref{eq:recon_loss} (encoder+decoder), full-latent code uploading (i.e., $\rho_i=1$), and one full-batch gradient-descent step on the \emph{alignment} loss \eqref{eq:align_loss} (encoder only). We establish a linear (geometric) decrease of each phase-specific objective~\cite{bottou2018optimization}, and then combine them to obtain a round-wise alternating gradient method convergence rate~\cite{tseng2001convergence, bolte2014proximal}.
We provide two analysis frameworks: first, a convex convergence analysis for which we show stable convergence towards a minimum; second, a non-convex convergence analysis for which we show convergence towards a critical point in finite time.

\begin{algorithm}[ht!]
\caption{Global Centroid Alignment (GCA).}
\label{alg:gpa}

\begin{algorithmic}[1]
\REQUIRE Clients $\{\mathcal{D}_i\}_{i=1}^C$, AEs $\{(E_i,D_i)\}_{i=1}^C$; rounds $R$; local epochs $E=T_{\text{recon}}+T_{\text{align}}$; batch size $B$; upload fractions $\{\rho_i\in(0,1]\}$; GMM components $K$; learning rates $\eta_{\text{recon}},\eta_{\text{align}}$; $\gamma=1.0$; $w_{\max}=6.0$; $\epsilon=10^{-6}$
\ENSURE Trained autoencoders; broadcast latent centroids and support counts 
\end{algorithmic}

% \begin{multicols}{2}
\begin{algorithmic}[1]
\STATE Initialize $\{\btheta_{E_i},\btheta_{D_i}\}$ for all clients
\FOR{$r=1$ to $R$} %\comment{Global round $r$}
  % ===== Phase 1: Reconstruction =====
   \STATE Server samples participating clients $\mathcal{S}_r$  
  \newline
  \STATE \textcolor{red}{\textbf{Phase 1: Reconstruction (client-side)}}
  \FORALL{clients $i\in\mathcal{S}_r$ \textbf{in parallel}}
    \FOR{$e=1$ to $T_{\text{recon}}$}
      \FORALL{mini-batches $B_i=\{\x_b\}_{b=1}^{B}\subset\mathcal{D}_i$}
        \STATE $\z_b \gets E_i(\x_b)$,\quad $\hat{x}_b \gets D_i(\z_b)$  %\Comment{$\z_b\in\mathbb{R}^{d_z}$}
        \STATE $\mathcal{L}_{\text{recon}}^{(i)} \gets \frac{1}{B}\sum_{b=1}^{B}\|\x_b-\hat{x}_b\|_2^2$
        \STATE $\btheta_{E_i}\gets \btheta_{E_i}-\eta_{\text{recon}}\nabla_{\btheta_{E_i}}\mathcal{L}_{\text{recon}}^{(i)}$
        \STATE $\btheta_{D_i}\gets \btheta_{D_i}-\eta_{\text{recon}}\nabla_{\btheta_{D_i}}\mathcal{L}_{\text{recon}}^{(i)}$
      \ENDFOR
    \ENDFOR
  \ENDFOR
  \FORALL{clients $i\in\mathcal{S}_r$ \textbf{in parallel}}
    \STATE Select index set $\mathcal{I}_i$ with $|\mathcal{I}_i|=\lfloor\rho_i N_i\rfloor$
    \STATE $\z_i^{\uparrow}\gets \{E_i(x):x\in\mathcal{D}_i,\ \text{indices}\ \mathcal{I}_i\}$  %\Comment{Latents only}
    \STATE Upload $\z_i^{\uparrow}$ to server
  \ENDFOR
  \newline
  % ===== Phase 2: Aggregation =====
  \STATE \textcolor{blue}{\textbf{Phase 2: Aggregation (server-side)}}
  \STATE $Z\gets \bigcup_{i\in\mathcal{S}_r} \z_i^{\uparrow}=\{\z_n\}_{n=1}^{N_r}$,\quad
          $N_r\gets \sum_{i\in\mathcal{S}_r} |\z_i^{\uparrow}|$
  \STATE Fit clustering algorithm with $Z$
  \STATE Compute counts $N_k$ and centroids $\mathbf{\mu}_k$
  \STATE Broadcast $\{\mathbf{\mu}_k,N_k\}_{k=1}^{K}$ to clients in $\mathcal{S}_r$
  \newline
  % ===== Phase 3: Alignment =====
  \STATE \textcolor{purple}{\textbf{Phase 3: Alignment (client-side, encoder-only)}}
  \FORALL{clients $i\in\mathcal{S}_r$ \textbf{in parallel}}
    \STATE $\widehat N_k\gets\max(N_k,\epsilon)$, $\bar N \gets \frac{1}{K}\sum_{k=1}^{K}\widehat N_k$
    \STATE $\tilde w_k\gets\min\{w_{\max},(\bar N/\widehat N_k)^\gamma\}$, $w_k\gets\frac{\tilde w_k}{\frac{1}{K}\sum_{j=1}^{K}\tilde w_j+\epsilon}$
    \FOR{$e=1$ to $T_{\text{align}}$}
      \FORALL{mini-batches $B_i=\{\x_b\}_{b=1}^{B}$}
        \STATE $\z_b \gets E_i(\x_b)$
        \STATE $a_b \gets \arg\min_{k\in\{1,\dots,K\}} \|\z_b-\mathbf{\mu}_k\|_2$  %\Comment{Nearest prototype}
        \STATE $\mathcal{L}_{\text{align}}^{(i)} \gets \frac{1}{B}\sum_{b=1}^{B} w_{a_b}\,\|\z_b-\mathbf{\mu}_{a_b}\|_2^2$
        \STATE $\btheta_{E_i}\gets \btheta_{E_i}-\eta_{\text{align}}\nabla_{\btheta_{E_i}}\mathcal{L}_{\text{align}}^{(i)}$  %\Comment{Encoder-only}
      \ENDFOR
    \ENDFOR
  \ENDFOR
\ENDFOR
\end{algorithmic}
% \end{multicols}
    
\end{algorithm}

\subsection{Convex Full-batch Convergence}\label{convex-conv}

\paragraph{Notation and Objectives.}
For client $i$, let $\theta_{E}\in\mathbb{R}^{p_E}$ and $\theta_{D}\in\mathbb{R}^{p_D}$ denote the parameters of the encoder $E_{\theta_E}(\cdot)$ and decoder $D_{\theta_D}(\cdot)$, respectively.
The reconstruction objective is $F(\theta_{E},\theta_{D})$ as in \eqref{eq:recon_loss}, and the alignment objective $G(\theta_{E};\mathcal{M},\mathbf{w})$ as in \eqref{eq:align_loss}.
%\harry{Have we defined $\mathbf{w}$ yet?}

\paragraph{Assumptions.}
We make smoothness/regularity assumptions to analyze alternating gradient methods around a local basin. \\
$\bullet$ \textbf{(A1) $L_F$-smoothness of $F$.} $F$ has $L_F$-Lipschitz continuous gradients in a neighborhood $\mathcal{B}$ of the iterates: 
\begin{equation*}
    \|\nabla F(\vartheta)-\nabla F(\vartheta')\|\le L_F\|\vartheta-\vartheta'\|
\end{equation*}
for all $\vartheta=(\theta_E,\theta_D),\vartheta'\in\mathcal{B}$.
\\
$\bullet$ \textbf{(A2) Polyak-{\L}ojasiewicz (P{\L}) condition for $F$.} There exists $\mu_F>0$ such that for all $(\theta_E,\theta_D)\in\mathcal{B}$:
\begin{equation*}
\frac{1}{2}\|\nabla F(\theta_E,\theta_D)\|^2 \ge \mu_F\Bigl(F(\theta_E,\theta_D)-F^\star\Bigr), \qquad F^\star\triangleq\inf_{\mathcal{B}}F.
\end{equation*} 
$\bullet$ \textbf{(A3) Fixed latent centroids and stable assignments (piecewise smoothness/P{\L}).}
Across all global rounds, treat $(\mathcal M,\mathbf w)$ as fixed constants during the alignment epoch. Let the nearest-centroid assignment be $a(\x)$ (cf.$a_b$ in \eqref{eq:nearest_latent}).
Assume there exists an open set (a fixed-assignment region)
\begin{equation*}
\mathcal R \triangleq \bigl\{\theta_E:a(\x;\theta_E)=a(\x;\theta_E^{+})\text{for all }x\in\mathcal D_i\bigr\}    
\end{equation*}
containing the encoder iterates during the alignment step.
On $\mathcal R$, the alignment objective $G(\cdot;\mathcal M,\mathbf w)$
is $L_G$-smooth and satisfies a P{\L} inequality with parameter $\mu_G>0$:
\begin{equation*}
\frac{1}{2}\bigl\|\nabla_{\theta_E} G(\theta_E;\mathcal M,\mathbf w)\bigr\|^2
\ge
\mu_G \bigl(G(\theta_E;\mathcal M,\mathbf w)-G^\star\bigr),
\end{equation*}
where \(G^\star \triangleq \min_{\theta_E\in\mathcal R} G(\theta_E;\mathcal M,\mathbf w)\)~\cite{karimi2016linear}. \\
% \textcolor{blue}{[C: What is the fixed-assignment region? $G^\star$ should change depending on the latent codes that get sent to the server (which you are explicitly saying do change from round to round). Or is $G^\star$ defined with respect to some optimal latent codes that do not change from round to round?]} \harry{My understanding is that G* is interpreted as the minimum within an area such as cluster assignments of a given set of latent codes would not change. After further thought, it might not be the most reasonable formulation since \label{eq:PL-G} is not dependent on the latents at any step. Alternatively, we can start with a set $(\mathcal{M}^\star, \mathbf{w}^*)$ redefined as the set of optimal GMM centroids when transforming the data distribution from $\mathbb{R}^{d_x}\to\mathbb{R}^{d_z}$. Or putting the centroids of the data distribution in space $\mathbb{R}^{d_x}$ directly through $E(x_b;\theta_E)$ to get our $\mu_k^*$.} \\
$\bullet$ \textbf{(A4) Bounded, mean-one weights.}
 The inverse-count weights are normalized as in \eqref{eq:weights}, satisfy $\frac{1}{K}\sum_k w_k=1$, and are bounded: 
\begin{equation*}
    0<w_{\min}\le w_k\le w_{\max}<\infty.
\end{equation*}
%(This prevents any single prototype from dominating the gradient scale.)
$\bullet$ \textbf{(A5) Fixed latent centroids and uniform latent proximity.}
There exists a fixed latent-centroid set $\mathcal{M}=\{\mu_k\}_{k=1}^K$ such that
\begin{equation*}
\sup_{\mathbf{x}}\min_{k}\big\|E_{\theta_E}(\x)-\mu_k\big\| \le\varepsilon_\mu.
\end{equation*}
That is, every encoding lies within an $\varepsilon_\mu$-neighborhood of some centroid; if the minimum separation
$\gamma\triangleq \min_{k\neq j}\|\mu_k-\mu_j\|$ satisfies $\varepsilon_\mu<\gamma/2$, nearest–centroid assignments are unique and stable (consistent with (A3)).

\paragraph{Phase-wise descent.}
For any $L$-smooth $H$ and update $\theta^+ = \theta - \eta g(\theta)$, yields the standard linear (geometric) one-step decrease
\begin{equation}
\begin{aligned}
H&(\theta^+) - H^\star \le (1-\eta^{(r)}\mu)\big(H(\theta)-H^\star\big),   \qquad 0<\eta^{(r)}\le \frac{1}{Lr}.
\label{eq:exact-descent-PL}
\end{aligned}
\end{equation}

\begin{lemma}[Reconstruction phase contraction]
\label{lem:recon}
For the full-batch GD update
$
\Theta^{+} = \Theta - \eta_{\text{recon}}^{(r)}\nabla F(\Theta),
$
one step satisfies
\begin{equation*}
F(\Theta^{+})-F^\star \le (1-\eta_{\text{recon}}^{(r)}\mu_F)\big(F(\Theta)-F^\star\big),
\end{equation*}
where $\Theta$ is the start-of-round parameter and $\Theta^{+}$ is the parameter after the reconstruction step. \\
\noindent\emph{Proof.}
Apply \eqref{eq:exact-descent-PL} to $H=F$, $L=L_F$, $\mu=\mu_F$. \hfill  $\square$ \end{lemma}

\begin{lemma}[Alignment phase contraction with nearest assignments and mean-one weights]
\label{lem:align}
Fix the broadcast $(\mathcal{M},\mathbf{w})$ within the round and assuming (A3)-(A5), the encoder-only update satisfies
\begin{equation}
\begin{aligned}
&G(\theta_E^{++};\mathcal{M},\mathbf{w})-G^\star \le
(1-\eta_{\text{align}}^{(r)}\mu_G)\bigl(G(\theta_E^{+};\mathcal{M},\mathbf{w})-G^\star\bigr) 
+ c_1 \eta_{\text{align}}^{(r)} \varepsilon_\mu^2,
\label{eq:align-one-step}
\end{aligned}
\end{equation}
where $c_1$ is only bounded by (A3)–(A4); one may use $c_1=\mu_G w_{\max}$. \\
\noindent\emph{Proof.}
On the fixed-assignment region from (A3), $G(\cdot;\mathcal M,\mathbf w)$ is $L_G$-smooth and $\mu_G$-P{\L}, hence the exact gradient step yields
\[
G(\theta_E^{++})-G^\star
 \le 
(1-\eta_{\text{align}}^{(r)}\mu_G)\bigl(G(\theta_E^{+})-G^\star\bigr).
\]
By (A5), every encoding lies within $\varepsilon_\mu$ of some centroid, hence
\begin{equation*}
\begin{aligned}
G&(\theta_E;\mathcal M,\mathbf w)  =\frac{1}{|\mathcal D_i|}\sum_x w_{a(x)}\|E_{\theta_E}(x)-\mu_{a(x)}\|^2
\le w_{\max}\varepsilon_\mu^2
\end{aligned}
\end{equation*}
for all iterates in the region; in particular, \(G^\star\le w_{\max}\varepsilon_\mu^2\).
Adding the nonnegative slack \(\eta_{\text{align}}^{(r)}\mu_G G^\star\) to the right-hand side and using the bound on \(G^\star\) gives \eqref{eq:align-one-step}. \hfill  $\square$ \end{lemma} 
% \noindent\emph{Proof.}
% View the alignment step as an \emph{inexact} gradient step for the population objective
% $H(\theta_E)\triangleq G(\theta_E;\mathcal{M},\mathbf{w})$, taken in the direction
% $-\nabla_{\theta_E}G(\theta_E;\mathcal{M},\mathbf{w})$. Let
% \[
% \delta(\theta_E)\triangleq \nabla_{\theta_E}G(\theta_E;\mathcal{M},\mathbf{w})
% -\nabla_{\theta_E}G(\theta_E;\mathcal{M}^\star,\mathbf{w}).
% \]
% On the stable-assignment cell, using the chain rule,
% \[
% \delta(\theta_E)
% =\frac{2}{|\mathcal{D}_i|}\sum_x w_{a(x)} J_E(x;\theta_E)^\top\big(\mu^\star_{a(x)}-\mu_{a(x)}\big),
% \]
% hence $\|\delta(\theta_E)\|\le 2w_{\max}J_{\max}\varepsilon_\mu =: c_\delta \varepsilon_\mu$ by (A4) and bounded encoder Jacobian \textcolor{blue}{[C: Boundedness of the encoder Jacobian needs to be explicitly stated as an assumption.]}.
% Applying Young's inequality to the cross term and then the P{\L} inequality for $H$ yields
% \[
% H(\theta_E^{++})-H^\star
% \le (1-\eta_{\text{align}}^{(r)}\mu_G)\big(H(\theta_E^{+})-H^\star\big)
% + \eta_{\text{align}}^{(r)}\|\delta(\theta_E^{+})\|^2,
% \]
% which, using $\|\delta\|^2\le c_\delta^2\varepsilon_\mu^2$ and setting $c_1:=c_\delta^2$, gives \eqref{eq:align-one-step}~\cite{devolder2014first}.

\paragraph{Across-Round Convergence of the Two-Phase Schedule.}
\label{subsec:round-convergence}
We now show that iterating GCA rounds converges, even though each round alternates between two different objectives.

Along with (A1)-(A5), assume the local curvature conditions in the basin of interest: \\
\noindent$\bullet$ \textbf{(A6)} $F$ is $\mu_F$-strongly convex and $L_F$-smooth on the basin containing the iterates. \\
\noindent$\bullet$ \textbf{(A7)} With fixed latent assignments, $G(\cdot;\mathcal M,\mathbf w)$ is $\mu_G$-strongly convex
and $L_G$-smooth in $\theta_E$. 
%\textcolor{blue}{[C: Does this mean the assignments are fixed and $\mathcal{M}$ does not change across rounds? Then you do not need to consider $\mathcal{M}^*$ and can remove Lemma 2.]} \harry{I think assignments not changing across rounds is going to need some further thinking about how we handle situations where a latent is ``mis-categorized'' from its underlying generating distribution}\\
Under (A6), for any such $H\in\{F,G\}$ we have the gradient-suboptimality sandwich
\begin{equation}
2\mu_H\big(H(\theta)-H^\star\big)\le\|\nabla H(\theta)\|^2\le2L_H\big(H(\theta)-H^\star\big).
\label{eq:grad-sandwich}
\end{equation}
Let $\Theta \triangleq (\theta_E,\theta_D)$ and define suboptimalities
$\Delta_F(\Theta) \triangleq F(\Theta)-F^\star$ and
$\Delta_G(\theta_E) \triangleq G(\theta_E;\mathcal M,\mathbf w)-G^\star$.
Within one round, recall:
\begin{equation*}
\begin{aligned}
&\Theta \xrightarrow[\text{recon}]{\eta_{\text{recon}}^{(r)}}\Theta^+=(\theta_E^+,\theta_D^+) \quad
\text{ and } \quad
\Theta^+ \xrightarrow[\text{align}]{\eta_{\text{align}}^{(r)}}\Theta^{++}=(\theta_E^{++},\theta_D^+)
.
\end{aligned}
\end{equation*}
\begin{lemma}[Effect of a reconstruction step on $G$]
\label{lem:G-after-F}
Run one full-batch GD step on $F$. Then, for any $\zeta>0$,
\begin{equation}
\begin{aligned}
\Delta_G(\theta_E^+) \le&
\Delta_G(\theta_E) 
 +  \eta_{\text{recon}}^{(r)}\frac{L_G}{\zeta} \Delta_G(\theta_E) 
 +  \eta_{\text{recon}}^{(r)} \left(\zeta L_F\right)\Delta_F(\Theta) 
 + O((\eta_{\text{recon}}^{(r)})^2).
\label{eq:G-after-F-explicit}
\end{aligned}
\end{equation}
where $\Delta_G(\theta_E) \triangleq G(\theta_E;\mathcal M,\mathbf w)-G^\star$. \\
\noindent\emph{Proof.} By $L_G$-smoothness on the encoder coordinate,
\[
G(\theta_E^+) \le G(\theta_E) + \big\langle \nabla_{\theta_E}G(\theta_E), \theta_E^+-\theta_E\big\rangle
+ \frac{L_G}{2}\|\theta_E^+-\theta_E\|^2.
\]
Since $\theta_E^+-\theta_E=-\eta_{\text{recon}}^{(r)}\nabla_{\theta_E}F(\Theta)$, apply Young's inequality to the cross term and use \eqref{eq:grad-sandwich} and $\|\nabla_{\theta_E}F\|\le\|\nabla F\|$ to obtain \eqref{eq:G-after-F-explicit}. \hfill $\square$ \end{lemma}

\begin{theorem}[Effect of an alignment step on $F$]
\label{lem:F-after-G-fixed}
Run one encoder-only alignment step 
\begin{equation*}
\theta_E^{++}=\theta_E^{+}-\eta_{\text{align}}^{(r)}\nabla_{\theta_E}G(\theta_E^{+};\mathcal M,\mathbf w).
\end{equation*}
Then, for any $\zeta'>0$,
\begin{equation}
\label{eq:F-after-G-explicit-fixed}
\begin{aligned}
\Delta_F(\Theta^{++})
& \le 
\Delta_F(\Theta^{+}) 
 + \eta_{\text{align}}^{(r)}\frac{L_F}{\zeta'} \Delta_F(\Theta^{+}) 
 + \eta_{\text{align}}^{(r)}\big(\zeta' L_G\big) \Delta_G(\theta_E^{+})
 + O((\eta_{\text{align}}^{(r)})^2),
\end{aligned}
\end{equation}
where $\Delta_F(\Theta) \triangleq F(\Theta)-F^\star$. \\
\noindent\emph{Proof.}
By $L_F$-smoothness of $F$ in the encoder coordinate,
\begin{equation*}
\begin{aligned}
F(\Theta^{++}) \le F(\Theta^{+}) 
&+\big\langle \nabla_{\theta_E}F(\Theta^{+}), \theta_E^{++}-\theta_E^{+}\big\rangle 
+\frac{L_F}{2}\|\theta_E^{++}-\theta_E^{+}\|^2.
\end{aligned}   
\end{equation*}
Substitute $\theta_E^{++}-\theta_E^{+}=-\eta_{\text{align}}^{(r)}\nabla_{\theta_E}G(\theta_E^{+})$ and apply Young's inequality to get
\begin{equation*}
\begin{aligned}
F(\Theta^{++})-F^\star
\le& \Delta_F(\Theta^{+})
+ \eta_{\text{align}}^{(r)}\frac{1}{2\zeta'}\|\nabla_{\theta_E}F(\Theta^{+})\|^2  + \eta_{\text{align}}^{(r)}\frac{\zeta'}{2}\|\nabla_{\theta_E}G(\theta_E^{+};\mathcal M,\mathbf w)\|^2
+ \frac{L_F}{2}(\eta_{\text{align}}^{(r)})^2\|\nabla_{\theta_E}G(\cdot)\|^2.
\end{aligned}
\end{equation*}
Using the gradient-suboptimality bounds (from strong convexity/smoothness in (A6)–(A7)),
\begin{equation*}
\begin{aligned}
& \|\nabla_{\theta_E}F(\Theta^{+})\|^2\le 2L_F\Delta_F(\Theta^{+}) \quad \text{ and } \quad \|\nabla_{\theta_E}G(\theta_E^{+})\|^2\le 2L_G\Delta_G(\theta_E^{+}),
\end{aligned}
\end{equation*}
yields \eqref{eq:F-after-G-explicit-fixed}. \hfill $\square$
\end{theorem}
For convenience, we write
$
\alpha_r \triangleq\eta_{\text{recon}}^{(r)}\frac{L_G}{\zeta}$, $
\beta_r \triangleq\eta_{\text{recon}}^{(r)}\zeta L_F$, $
\alpha_a \triangleq\eta_{\text{align}}^{(r)}\frac{L_F}{\zeta'}$, and $ 
\beta_a\triangleq\eta_{\text{align}}^{(r)}\zeta' L_G.
$
Then,
\begin{equation*}
\begin{aligned}
&\Delta_F(\Theta^{++}) \le 
\Big[1-\mu_F\eta_{\text{recon}}^{(r)}  + \alpha_a+ \beta_a\beta_r - \alpha_a\mu_F\eta_{\text{recon}}^{(r)}\Big]\Delta_F(\Theta) 
+\beta_a(1+\alpha_r)\Delta_G(\theta_E),\\
&\Delta_G(\theta_E^{++})\le 
\Big[1-\mu_G\eta_{\text{align}}^{(r)} + \alpha_r - \mu_G\eta_{\text{align}}^{(r)}\alpha_r\Big]\Delta_G(\theta_E) 
+\Big[\beta_r - \mu_G\eta_{\text{align}}^{(r)}\beta_r\Big]\Delta_F(\Theta)
+c_1\eta_{\text{align}}^{(r)}\varepsilon_\mu^2.
\end{aligned}
\end{equation*}
Since $\alpha_r,\beta_r,\alpha_a,\beta_a=O(r^{-1})$, products like $\alpha_a\mu_F\eta_{\text{recon}}^{(r)}$, $\beta_a\beta_r$, $\mu_G\eta_{\text{align}}^{(r)}\alpha_r$, and $\mu_G\eta_{\text{align}}^{(r)}\beta_r$ are $O(r^{-2})$ and can be absorbed.
Summing the two displays gives, with $S^{(r)}\triangleq \Delta_F^{(r)}+\Delta_G^{(r)}$,
\begin{equation}
\begin{aligned}
\label{eq:one-round-pre}
S^{(r+1)}
\le&\Big[1-\mu_F\eta_{\text{recon}}^{(r)} +\alpha_a+ \beta_r\Big]\Delta_F^{(r)} 
+\Big[1-\mu_G\eta_{\text{align}}^{(r)} + \alpha_r + \beta_a\Big]\Delta_G^{(r)} 
+c_1\eta_{\text{align}}^{(r)}\varepsilon_\mu^2+O(r^{-2}).
\end{aligned}
\end{equation}

We pick $\zeta=\sqrt{L_G/L_F}$ and $\zeta'=\sqrt{L_F/L_G}$, which minimize the cross coefficients. Then
\[
\alpha_r+\beta_r  =  2\sqrt{L_F L_G}\,\eta_{\text{recon}}^{(r)},\quad\text{ and }\quad
\alpha_a+\beta_a  =  2\sqrt{L_F L_G}\,\eta_{\text{align}}^{(r)}.
\]

Using $\alpha_r,\beta_r\le \alpha_r+\beta_r$ and $\alpha_a,\beta_a\le \alpha_a+\beta_a$ in
\eqref{eq:one-round-pre} and dropping the negative terms
$-\mu_F\eta_{\text{recon}}^{(r)}\Delta_F^{(r)}$ and
$-\mu_G\eta_{\text{align}}^{(r)}\Delta_G^{(r)}$ yields
\begin{equation*}
\begin{aligned}
S^{(r+1)} &\le \Bigl(1+\gamma_r\Bigr)S^{(r)}  + c_1\,\eta_{\text{align}}^{(r)}\,\varepsilon_\mu^2 + O(r^{-2}),
\end{aligned}
\end{equation*}
where 
$\gamma_r \triangleq 2\sqrt{L_F L_G}\bigl(\eta_{\text{recon}}^{(r)}+\eta_{\text{align}}^{(r)}\bigr)$.

We assume bounded objective gaps. \\
\noindent$\bullet$ \textbf{(A8)} There exist $S_{\max}<\infty$ and $r_0\in\mathbb N$ such that
\[
\sup_{r\ge r_0} S^{(r)}\le S_{\max}, \quad \text{ where } \quad S^{(r)}\triangleq \Delta_F^{(r)}+\Delta_G^{(r)}.
\] \\
Under (A8), from the one–step bound above we have, for all $r\ge r_0$,
\begin{equation*}
\begin{aligned}
S^{(r+1)}-S^{(r)}
%&\le \gamma_r\,S^{(r)} + c_1\,\eta_{\text{align}}^{(r)}\,\varepsilon_\mu^2 + O(r^{-2}) \\
&\le \gamma_r\,S_{\max} + c_1\,\eta_{\text{align}}^{(r)}\,\varepsilon_\mu^2 + O(r^{-2}).   
\end{aligned}
\end{equation*}
 Using the admissible decays $\eta_{\text{recon}}^{(r)}$ and 
$\eta_{\text{align}}^{(r)}$ gives
\begin{equation*}
\begin{aligned}
S^{(r+1)}-S^{(r)} 
\le& \frac{2\!\Bigl(\sqrt{\tfrac{L_G}{L_F}} 
 +\sqrt{\tfrac{L_F}{L_G}}\Bigr)\,S_{\max}}{r} 
 + \frac{\tfrac{c_1}{L_G}\,\varepsilon_\mu^2}{r} + O(r^{-2}), 
\end{aligned}
\end{equation*}
We can also provide a lower bound for $S^{(r + 1)}$ in a similar way. 

Therefore, $\lim_{r\to\infty}\bigl(S^{(r+1)}-S^{(r)}\bigr)=0$: the sequence $\{S^{(r)}\}$ is \emph{stable} in the sense that, by (A8), it remains uniformly bounded.

\subsection{Non-convex Full-batch Convergence}\label{non-convex-conv}

To show convergence to a stationary point in the error surface, we adapt the analysis of %\textcolor{blue}{[C: instead of ``use'' should we say ``adapt the analysis of''? Otherwise it sounds like we are using a completely different algorithm from Algorithm 1, and we'd have to immediately explain why AFB/PALM is a good approximation of the GRA algorithm]}
the Alternating Forward-Backward (AFB) algorithm~\cite{Frankel_2014}, also known as PALM~\cite{bolte2014proximal}, as a benchmark to evaluate GCA. AFB is a proximal gradient descent method and treats the optimization problem
\begin{equation}
    \minimize_{x_i\in X_i} h(x_1, \hdots, x_p) = \sum^p_{i=1} f_i(x_i) + g(x_1, \hdots, x_p)
\end{equation}
where the independent variables are split into blocks $x_i$ and alternates in locally optimizing sub-objectives over blocks of variables and updating them, i.e.
\begin{equation}
\begin{aligned}
    x^{(r+1)}_i \in \prox^{f_i}_{\alpha_{i, r}} \Big( x^{(r)}_i& - \frac{1}{\alpha_{i, r}} \nabla_i g(x^{{(r+1)}}_1, 
    \hdots, x^{{(r+1)}}_{i-1}, x^{(r)}_i, \hdots, x^{(r)}_p) \Big)
\end{aligned}
\end{equation}\label{eq:afb-update}
where the proximal map consists of
\begin{equation}
    \prox^{f}_{\alpha} (\x) \triangleq \argmin_{\mathbf{y}} \left\{ f(\y) + \frac{\alpha}{2}||\y-\x||^2 \right\}
\end{equation}
with inverse learning rate $\alpha$. 
% \textcolor{red}{[J: here x and y in prox should be vector?]}
The sequence generated by AFB is known to be of finite length if the objective function $h$ satisfies the Kurdyka-{\L}ojasiewicz (K{\L}) condition (defined below) everywhere in the domain of $\partial h$ (known as a K{\L} function), thus converging towards a critical point in $h$. The only other assumptions necessary for AFB are that $g$ is $L_g$-smooth and $f_i$ are proper lower semicontinuous functions, where $g, f_i$ \emph{do not have to be convex in any way}. 
% \harry{Summarize convergence results of AFB here?}

\noindent\textbf{Notation and Objectives.$\quad$}
For client $i$, let $\theta_{E_i}\in\mathbb{R}^{p_E}$ and $\theta_{D_i}\in\mathbb{R}^{p_D}$ denote the parameters of the encoder $E_i(\cdot) \triangleq E_i(\;\cdot\;; \theta_{E_i})$ and decoder $D_i(\cdot) \triangleq D(\;\cdot\;; \theta_{D_i})$, respectively.
For simplicity, we concatenate the model parameters at each client $i$ into a single $\theta \triangleq \{\theta_{E_i},\theta_{D_i}\}$ and the clustering parameters into a single $\mu \triangleq \mathcal{M}$.
We rename the reconstruction objective as in \eqref{eq:recon_loss} using $F(\theta) = \sum_i \mathcal{L}^{(i)}_{\text{recon}}(\btheta_{E_i},\btheta_{D_i} ;\mathcal{D}_i)$, and we define a co-objective $G(\theta, \mu)$ which encompasses two requirements:
% \textcolor{red}{[J: Earlier $F$ denotes reconstruction and $G$ denotes alignment. Renaming reconstruction to $G$ and then defining $F(\theta,\mu)=G(\theta)+H(\theta,\mu)$ will confuse readers.}\\
\newline
\noindent$\bullet$ When the encoder parameters are fixed, $G(\mu; \theta_{E}%, \mathcal{D}
    )$ should serve as a clustering objective measuring proximity between $\mu$, a set of $K$ points, and the ``true'' centroids of a GMM fitted onto the underlying data distribution projected onto the latent space through the encoder with parameters $\theta_{E}$.
    %\textcolor{violet}{Jong-Ik: when you define $G$, I think we need one equation line to clarify that $ G$ contains $G_{cl}$ and $L_{align}$ if possible.}
    %The soft weights $\mathbf{w}(\mu; \theta_E, \mathcal{D})$ \textcolor{blue}{[C: from the alignment phase? It's not clear how these weights influence the clustering objective]} are computed in relation to the latent representation of the data set using the encoders $E_i(\cdot)$, thus the inclusion of the dataset $\mathcal{D} \triangleq \{\mathcal{D}_i\}$ as a parameter.
\newline
\noindent$\bullet$ When the centroid locations are fixed, $G(\theta_{E}; \mu)$, the alignment objective for each client $i$, $\mathcal{L}^{(i)}_{\text{align}}(\theta_{E_i};\mu,\mathcal{D}_i)$ as per~\eqref{eq:align_loss}, such that the partial gradient of $G$ on the encoder parameters
    \begin{equation*}
        \nabla_{\theta_E} G (\theta, \mu) \triangleq
        \left\{ \frac{\partial G(\theta, \mu)}{\partial \theta_{E_i}} \right\} \triangleq 
        \left\{ \nabla\mathcal{L}^{(i)}_{\text{align}}(\theta_{E_i};\mu,\mathcal{D}_i) \right\}
    \end{equation*}
    is equal by definition to the set of gradients of the alignment loss on each client. Moreover, we assert that since the alignment step does not modify the decoder parameters, the partial gradient of the co-objective $G$ on the decoder parameters is zero and we define the partial gradient of $G$ on the model parameters in general as:
    \begin{align*}
        \nabla_{\theta_D} G (\theta, \mu) & \triangleq
        \left\{ \frac{\partial G(\theta, \mu)}{\partial \theta_{D_i}} \right\} = 0 \\
        \nabla_{\theta} G (\theta, \mu) & \triangleq \Bigl( \nabla_{\theta_E} G (\theta, \mu), \nabla_{\theta_D} G (\theta, \mu) \Bigr) \\
        &= \left(\left\{ \nabla\mathcal{L}^{(i)}_{\text{align}}(\theta_{E_i};\mu,\mathcal{D}_i) \right\}, 0 \right)
    \end{align*}

% \harry{Relabel below to Lemma?}
\begin{lemma}[Effective co-objective]
The requirements for the co-objective $G(\theta, \mu)$ can be satisfied with the alignment loss of GCA alone:
\begin{equation*}
    \mathcal{L}^{(i)}_{\text{align}}(\mathcal{M},\mathbf{w};\theta_{E_i},\mathcal{D}_i) = \frac{1}{|\mathcal{D}_i|}\sum_{d=1}^{|\mathcal{D}_i|} w_{a_{i,d}}\,\big\|E_i(\x_{i,d})-\mathbf{\mu}_{a_{i,d}}\big\|_2^2
\end{equation*}
where $a_{i,d}$ indexes the latent centroid assigned to data point $d$ at client $i$ is assigned.
\end{lemma}
\begin{proof}
Recalling  \eqref{eq:nearest_latent}, \eqref{eq:weights}, and \eqref{eq:align_loss}, note that since computing GMM using the EM approximately solves %\textcolor{blue}{[C: is $a_{i,d}$ here the assignment of data point $d$ at client $i$ to a prototype? Also, we should rephrase this to ``approximately solves'' since EM can't guarantee a solution]}
\begin{equation*}
\begin{aligned}
    \min_{\mathcal{M},\mathbf{w},a_b} &\mathcal{L}^{(i)}_{\text{align}}(\mathcal{M},\mathbf{w};\theta_{E_i},\mathcal{D}_i) 
    = \min_{\mathcal{M},\mathbf{w},a_b}  \frac{1}{|\mathcal{D}_i|}\sum_{d=1}^{|\mathcal{D}_i|} w_{a_{i,d}}\,\big\|E_i(\x_{i,d})-\mathbf{\mu}_{a_{i,d}}\big\|_2^2
\end{aligned}
\end{equation*}
and that $\nabla_\mu \mathcal{L}^{(i)}_{\text{align}}(\theta_{E},\mu;\mathcal{D}_i)$ is piecewise $L_H$-continuous, the properties of objective $G_{\text{cl}}(\mu; \theta_{E}%, \mathcal{D}
)$ can be fulfilled using $\mathcal{L}_{\text{align}}$. Since the alignment loss satisfies both the clustering objective and the alignment objective (by definition above), the alignment loss of GCA on its own can be used as the co-objective. \qed
\end{proof}

Therefore, we define the co-objective $G(\theta, \mu)$ as
\begin{equation}
\begin{aligned}
    G(\theta, \mu)& \equiv \sum_i \mathcal{L}^{(i)}_{\text{align}}(\theta_{E_i},\mu;\mathcal{D}_i)\\
    &= \sum_i \frac{1}{|\mathcal{D}_i|}\sum_{d=1}^{|\mathcal{D}_i|} w_{a_{i,d}}\,\big\|E_i(\x_{i,d})-\mathbf{\mu}_{a_{i,d}}\big\|_2^2,
    \qquad a_{i,d} \in \argmin_k\ \|E_i(\x_{i,d})-\mathbf{\mu}_k\|_2, 
\label{eq:co-objective}
\end{aligned}
\end{equation}
\begin{align*}
    N_k = \sum_i \sum_{d=1}^{|\mathcal{D}_i|}\indicator(a_{i,d} = k), 
    \quad
    \tilde w_k=\frac{\frac{1}{K}\sum_{j=1}^{K}N_j}{N_k},  \quad
    w_k=\frac{\tilde w_k}{\frac{1}{K}\sum_{j=1}^{K}\tilde w_j}.
\end{align*}
% \textcolor{red}{[J: I am not sure whether we can redefine H as alignment loss. If this is true, then why don't we use simple gradient descent convergence analysis instead of below? Because now we only have two smoothed objectives.]}
% \textcolor{red}{[J: $a_{i,d}$ and $w_k(N_k)$ introduce non-smoothness/discontinuities when assignments change. I think claims like $H$ is $L_H$-smooth only hold within a fixed-assignment region or under soft assignments. Either (i) add a fixed-assignment-region assumption (as in Sec.~\ref{convex-conv}), or (ii) use soft assignments/soft counts to make $H$ globally smooth. Maybe piecewise smoothness within a fixed-assignment region?]}
We define the global objective function as $H(\theta, \mu) = F(\theta) + G(\theta, \mu)$.

\paragraph{Assumptions.}
As with AFB, we make few assumptions on the objective function.

$\bullet$ \textbf{(A1+) Properties of objective function.} $G$ is $L_G$-smooth and $F$ is $L_F$-smooth and a proper lower semicontinuous function, and note that $F, G$ do not have to be convex in any way.

$\bullet$ \textbf{(A2+) Kurdyka-{\L}ojasiewicz (K{\L}) condition for $G$.} The K{\L} property for a function $g$ is satisfied at a given $x \in \dom(g)$ if there exist a neighborhood $X$ around $x$, an $\eta \in (0, +\infty]$, and a concave and continuous function $\varphi$ which satisfy $\varphi(0)=0$, continuous at $0$, $C^1$ on $(0, \eta)$, and $\varphi'>0\; \forall s \in (0, \eta)$, such that $\forall u \in U \cap [g(x) < g(u) < g(x) + \eta]$, the following condition holds:
\begin{equation}
    \varphi'(g(u) - g(x)) \dist(0, \partial g(u)) \geq 1
\end{equation}
We assume that $G$ is a K{\L} function, that is it satisfies the K{\L} condition over all points in the domain of $\partial G$.

\paragraph{Phase-wise descent.}
We then rearrange the phases in GCA (see Algorithm~\ref{alg:gpa}) and define an AFB-like method with variable blocks $(\theta, \mu)$ in where we repeat the following steps until convergence:
\begin{enumerate}
    \item $\mu^{(r+1)}$ takes the centroids computed from GMM based on latent codes computed using $E(\theta^{(r)})$ \\ (Phase 2 in GCA)
    \item $\theta^{(r+1)}$ takes two gradient steps, one descending on $G(\theta, \mu)$, then descending on $F(\theta)$;
    \begin{itemize}
        \item $\theta^{(r)+} = \theta^{(r)} - \eta^{(r)}_{\text{align}} \nabla_{\theta} G(\theta^{(r)}; \mu)$ as per Phase 3 in GCA, then
        \item $\theta^{(r+1)} = \theta^{(r)+} - \eta^{(r)}_{\text{recon}} \nabla F(\theta^{(r)+})$ as per Phase 1 in GCA. \\We define $g^{(r)}(\theta^{(r)}) = \eta^{(r)}_{\text{align}} \nabla_{\theta} G(\theta^{(r)}; \mu) + \eta^{(r)}_{\text{recon}} \nabla F(\theta^{(r)+})$ as the total gradient step taken. Note that Phase 1 from the subsequent step in as GCA is implemented is shifted to the current step under this arrangement.
    \end{itemize}
\end{enumerate}

We note differences between GCA and AFB. First, in the update of $\mu$, since we use the EM algorithm to update GMM centroids, we approximate the minimizer of $G(\mu;\theta)$ along $\mu$ with fixed $\theta$.
% \textcolor{violet}{Jong-Ik: Does ``effectively compute the minimizer of $G(\mu;\theta)$'' mean you will assume they will always find local optima?}

Though it may result in a lower $H(\theta^{(r)}, \mu^{(r+1)}_{\text{GCA}}) \leq H(\theta^{(r)}, \mu^{(r+1)}_{\text{AFB}})$ (where $\mu^{(r+1)}_{\text{GCA}}$ is the centroid location update computed througn EM in the GCA algorithm formulation, and $\mu^{(r+1)}_{\text{AFB}}$ is a theoretical update of $\mu$ computed as per the AFB algorithm)
, it does not guarantee a better $H(\theta^{(r+1)}_{\text{GCA}}, \mu^{(r+1)}_{\text{GCA}})$ than $H(\theta^{(r+1)}_{\text{AFB}}, \mu^{(r+1)}_{\text{AFB}})$. As such, we define the difference between the GCA update for $\mu$ and the AFB update (just gradient descent since there is no proximal gradient) as
\begin{equation}
    \delta^{(r)} = \mu^{(r+1)}_{\text{GCA}} - \mu^{(r)} + {\alpha^{(r)}_\mu}^{-1} \nabla_\mu G(\mu;\theta)
\end{equation}\label{eq:inner-afb-error}
where ${\alpha^{(r)}_\mu}^{-1}$ is an arbitrary descent rate.

For the model parameter update, we recall the AFB update \eqref{eq:afb-update}
\begin{equation*}
    \theta^{(r+1)}_{\text{AFB}} \in \prox^{F}_{\alpha^{(r)}_\theta} \left( \theta^{(r)} - {\alpha^{(r)}_\theta}^{-1} \nabla_{\theta} G(\theta^{(r)}; \mu) \right)
    % = \argmin_{x} \left\{ \langle \theta - \theta^{(r)}, \nabla_{\theta} H(\theta^{(r)}; \mu) \rangle + \frac{\eta^{(r)}_{\text{align}}}{2}||\theta - \theta^{(r)}||^2 + F(\theta) \right\}
\end{equation*}
where ${\alpha^{(r)}_\theta}^{-1}$ is another arbitrary descent rate. This update rule is different from the GCA update due to the replacement of the proximal map for $F(\theta)$ with a subsequent gradient descent step. We therefore define the difference between the GCA update for $\theta$ and the AFB update as
\begin{equation}
    \varepsilon^{(r+1)} = \theta^{(r+1)}_{\text{GCA}} - \theta^{(r+1)}_{\text{AFB}}
    % = \theta^{(r)} - \eta^{(r)}_{\text{align}} \nabla_{\theta} H(\theta^{(r)}; \mu) - \eta^{(r)}_{\text{recon}} \nabla F\Bigl(\theta^{(r)} - \eta^{(r)}_{\text{align}} \nabla_{\theta} H(\theta^{(r)}; \mu)\Bigr) - \prox^{F}_{\eta^{(r)}_{\text{align}}} \left( \theta^{(r)} - \eta^{(r)}_{\text{align}} \nabla_{\theta} H(\theta^{(r)}; \mu) \right)
\end{equation}\label{eq:outer-afb-error}

\begin{proposition}[GCA as an imperfect form of AFB]
Given the sequences ${\delta^{(r)}}$ and ${\varepsilon^{(r)}}$, we can model GCA as an AFB implementation with errors as per Section 4.2 of~\cite{Frankel_2014} and that GCA converges under typical training regimes with decaying learning rate.
\end{proposition}

\begin{proof}
For the series $||\theta^{(r+1)}_{\text{GCA}}-\theta^{(r)}_{\text{GCA}}||$ to converge, %\textcolor{blue}{[C: are these the GRA $\theta$s or the AFB ones?]}
 the above work states that the following condition (\textbf{HE}) must be met: there exists a $\sigma \in [0, +\infty), \rho \in (0, 1]$ with $\frac{\sigma+1}{\rho} < \ubar{\alpha}/L_G$ where $\ubar{\alpha} > 0, \min\{\alpha^{(r)}_\theta, \alpha^{(r)}_\mu\} \geq \ubar{\alpha} > L_G$ such that:
\begin{enumerate}
    \item $\left\lVert\varepsilon^{(r)}\right\rVert \leq \tfrac{\sigma}{2}\left\lVert\theta^{(r+1)}_{\text{AFB}} - \theta^{(r)}_{\text{AFB}}\right\rVert$ where $\theta^{(r)}_{\text{AFB}}$ are defined in relation to the value from the previous step as computed by GCA.
    \item $\left\lVert\delta^{(r)}\right\rVert \leq \tfrac{\sigma}{2}\left\lVert\mu^{(r+1)}_{\text{AFB}} - \mu^{(r)}_{\text{AFB}}\right\rVert$ where $\mu^{(r)}_{\text{AFB}} = \mu^{(r)} - {\alpha^{(r)}_\mu}^{-1} \nabla_\mu G(\mu;\theta)$ as per \eqref{eq:inner-afb-error} and defined in relation to the value from the previous step as computed by GCA.
    \item $\left\langle \varepsilon^{(r)}, \theta^{(r+1)}_{\text{AFB}} - \theta^{(r)}_{\text{AFB}} \right\rangle \leq \tfrac{1-\rho}{2}\left\lVert\theta^{(r+1)}_{\text{AFB}} - \theta^{(r)}_{\text{AFB}}\right\rVert^2$ and $\left\langle \delta^{(r)}, \mu^{(r+1)}_{\text{AFB}} - \mu^{(r)}_{\text{AFB}} \right\rangle \leq \tfrac{1-\rho}{2}\left\lVert\mu^{(r+1)}_{\text{AFB}} - \mu^{(r)}_{\text{AFB}}\right\rVert^2$.
\end{enumerate}

We remark that the above conditions can be met if we select an increasing sequence $\alpha^{(r)}_\theta > L_G$ and superior sequences $1/\eta^{(r)}_{\text{recon}} > 2\alpha^{(r)}_\theta\; \forall r$ and $1/\eta^{(r)}_{\text{align}} > 2\alpha^{(r)}_\theta\; \forall r$, the conditions on $\varepsilon^{(r)}$ can be met.

For the centroid step, recall the co-objective $G(\theta, \mu)$ based on the alignment loss in \eqref{eq:co-objective} which is entirely composed of squared $L_2$ norms when $\theta$ is fixed. We thus know that its gradient and smoothness in $\mu$ are analytically obtainable and that a single-step gradient descent size $\alpha^{(r)}_\mu$ can be computed at each round; as such, we know that there exists a sequence $\alpha^{(r)}_\mu$ for which $\left\lVert\delta^{(r)}\right\rVert$ is negligible.

\end{proof}

We thus know that GCA can be analyzed as an AFB-type algorithm with tractable errors, and following results in~\cite{Frankel_2014}, we know that it can converge towards a critical point of $H(\theta, \mu)$.
% \textcolor{red}{[J: I think convergence to a critial point does not imply this moves toward the recon optimum I think?]}
Note that this method shows a convergence towards the co-objective $H(\theta, \mu) = F(\theta) + G(\theta, \mu)$, not solely the reconstruction objective $F(\theta)$.%As GRA approaches a critical point of $H$, if it is applied to an application where we can assume local convexity around a local optimum of the reconstruction loss, then the constraints set by \textbf{HE} can be relaxed such that gradient descent for the reconstruction objective dominates and the parameters move towards an optimum in reconstruction loss alone.%, and the analysis in Appendix~\ref{convex-conv} holds.

\clearpage
\newpage
\section{Additional Theory for Server-Side Data Extraction Attacks}
\label{app:attack_theory}
This appendix provides supporting theoretical results for the two server-side extraction routes introduced in the \emph{Adversarial Attack Scenarios} section:
(i) DeepDream-style inversion against a \emph{white-box} victim autoencoder (standard FL), and
(ii) DeepDream-style inversion against a surrogate trained \emph{from scratch} on uploaded latent codes only (GCA).

\subsection{Why reconstruction objectives admit training-like fixed points (standard FL)}
\label{app:ae_vulnerability}

Let $E:\mathbb{R}^{d_x}\to\mathbb{R}^{d_z}$ and $D:\mathbb{R}^{d_z}\to\mathbb{R}^{d_x}$ be a trained autoencoder, and define the reconstruction map
\begin{equation}
T(\x)\ \triangleq\ D(E(\x)).
\end{equation}
In standard FL, the server-side DeepDream-style inversion in \eqref{eq:attack_deepdream_update} minimizes the reconstruction energy
\begin{equation}
r(\x)\ \triangleq\ \|\x-T(\x)\|_2^2
\;\;\text{over}\;\; \x\in\mathcal{X}.
\label{eq:app_r_def}
\end{equation}
We analyze the geometry of $r(\x)$ to explain why its minimizers and near-minimizers lie near the autoencoder's reconstruction set, which may reflect the training distribution when the autoencoder reconstructs its training data accurately.

Define the reconstruction residual
\begin{equation}
e(\x)\ \triangleq\ \x-T(\x)\in\mathbb{R}^{d_x},
\;\;\text{so that}\;\;
r(\x)=\|e(\x)\|_2^2.
\label{eq:app_residual_def}
\end{equation}
Assume $T$ is differentiable on $\mathcal{X}$ and denote its Jacobian by $J_T(\x)\in\mathbb{R}^{d_x\times d_x}$.
Then $e$ is differentiable with $J_e(\x)=I-J_T(\x)$, and by the chain rule,
\begin{equation}
\begin{aligned}
\nabla r(\x)
= 2\,J_e(\x)^\top e(\x) &= 2\,(I - J_T(\x))^\top \bigl(\x-T(\x)\bigr)
\\
& = 2\,(I - J_T(\x))^\top e(\x).
\label{eq:app_grad_r_exact}
\end{aligned}
\end{equation}

\paragraph{Fixed points and global minimizers.}
Define the set of fixed points (exact reconstructions)
\begin{equation}
\mathcal{F}\ \triangleq\ \{\x\in\mathcal{X}\,:\,T(\x)=\x\}.
\label{eq:app_fix_set_def}
\end{equation}
By definition, $r(\x)=\|\x-T(\x)\|_2^2\ge 0$ and
\begin{equation}
r(\x)=0 \quad\Longleftrightarrow\quad T(\x)=\x \quad\Longleftrightarrow\quad \x\in\mathcal F.
\label{eq:app_zero_iff_fix}
\end{equation}
Consequently, if $\mathcal F\neq\emptyset$ (i.e., $r$ attains value $0$ on $\mathcal X$), then
\begin{equation}
\min_{\x\in\mathcal X} r(\x)=0,
\qquad
\arg\min_{\x\in\mathcal X} r(\x)=\mathcal F.
\label{eq:app_fix_global_min}
\end{equation}
If $\mathcal F=\emptyset$, then $r(\x)>0$ for every $\x\in\mathcal X$. If the minimum is attained, its value is strictly positive, and its minimizers are the points with the smallest reconstruction residual.

\paragraph{Why fixed points can be training-like.}
If the autoencoder reconstructs its training records well, then for many training samples $\x_n$ we have $r(\x_n)\approx 0$.
In such regimes, the reconstruction map $T$ is close to the identity on the data manifold, and training samples (or nearby points on the learned manifold) behave as approximate fixed points of $T$.
Therefore, directly minimizing $r(\x)$ via gradient-based input optimization admits solutions that lie on (or near) the learned reconstruction set, which in turn is shaped by the training distribution.

%\textcolor{red}{[C: $\mathcal{F}$ hasn't been defined]} 
% \paragraph{Stationary points may exist beyond fixed points.}
% Equation~\eqref{eq:app_grad_r_exact} implies the stationarity condition
% \begin{equation}
% \nabla r(\x)=0
% \quad\Longleftrightarrow\quad
% (I-J_T(\x))^\top e(\x)=0.
% \label{eq:app_stationary_condition}
% \end{equation}
% Besides $e(\x)=0$ (fixed points), this can also occur when $e(\x)\neq 0$ lies in the left nullspace of $(I-J_T(\x))$.
% Thus, $r$ may have stationary points that are not fixed points, as is typical for nonconvex objectives. \textcolor{red}{[C: What is the point of this subsection? It seems to work against what you are arguing, namely that reconstruction objectives allow for reconstruction of training data. Without characterizing $J_T$, it's hard to know if this is a mathematical artifact or if it is practically significant. It is also very obvious from the definition of $r(\mathbf{x})$.]}

\subsection{Why latent-only flipping can still yield training-like reconstructions (GCA), and why it is harder than white-box inversion}
\label{app:flip_can_work_but_harder}

This subsection formalizes why the latent-only ``decoder--encoder flipping'' route in the \emph{Adversarial Attack Scenarios} section is fundamentally non-identifying: from uploaded latents alone, the server cannot uniquely determine the underlying private training records in input space. Consequently, unlike standard-FL white-box inversion, latent-only flipping cannot guarantee faithful reconstruction of client data without additional assumptions.

Let the (unknown) client autoencoder be $(E,D)$ and let the server observe only a finite set of uploaded latent codes 
\begin{equation}
Z\ \triangleq\ \{\z_n\}_{n=1}^{M}\subset\mathbb{R}^{d_z},
\qquad
\z_n = E(\x_n),
\label{eq:app_Z_def}
\end{equation}
for unknown private records $\{\x_n\}_{n=1}^{M}\subset\mathcal X$.
A latent-only attacker is any (possibly randomized) mapping $\mathcal A$ that takes $Z$ as input and outputs reconstructions
$\hat{\mathbf{x}}_n=\mathcal A_n(Z)\in\mathcal X$.
% \textcolor{red}{[C: in A.1, $x$ is bolded. Why isn't it bolded here?]}

\paragraph{What the flipping attacker optimizes.}
%\textcolor{red}{[C: This result implies that you are minimizing the autoencoder training loss in order to reconstruct the original training data. That process needs to be explained.]}
The flipping attack chooses a parametric surrogate pair $(\tilde D,\tilde E)$ and trains it \emph{only on $Z$} by minimizing a latent reconstruction objective of the form
\begin{equation}
\min_{\theta_{\tilde D},\theta_{\tilde E}}
\ \ \frac{1}{M}\sum_{n=1}^{M}
\bigl\|\z_n-\tilde E(\tilde D(\z_n))\bigr\|_2^2.
\label{eq:app_flip_train_obj}
\end{equation}
After training, it defines $\widetilde T=\widetilde D\circ\widetilde E$ and applies the same
input-space reconstruction objective used by standard FL. Importantly, \eqref{eq:app_flip_train_obj} constrains only the composition $\widetilde E\circ\widetilde D$ on the finite observed set $Z$ and never observes the corresponding $\mathbf x_n$.

\begin{proposition}[Latent-only indistinguishability up to isometries]
\label{prop:latent_indistinguishability}
Fix any bijection $\phi:\mathcal X\to\mathcal X$ that preserves the $\ell_2$ norm, i.e.,
$\|\phi(\mathbf{u})-\phi(\mathbf{v})\|_2=\|\mathbf{u}-\mathbf{v}\|_2$ for all $\mathbf{u},\mathbf{v}\in\mathcal X$.
Define transformed private records $\x_n'=\phi(\x_n)$ and a transformed encoder
\begin{equation}
E'(\x)\ \triangleq\ E(\phi^{-1}(\x)).
\end{equation}
Then the induced latent set is \emph{identical}:
\begin{equation}
\begin{aligned}
E'(\x_n') &= E(\phi^{-1}(\phi(\x_n))) \\
&= E(\x_n) =\z_n
\qquad \forall n\in[M].
\end{aligned}
\end{equation}
Moreover, if $D'(\z)\triangleq \phi(D(\z))$, then the autoencoder reconstruction error is preserved:
\begin{equation}
\begin{aligned}
\|\x_n'-D'(E'(\x_n'))\|_2
&=
\|\phi(\x_n)-\phi(D(E(\x_n)))\|_2 \\
&=
\|\x_n-D(E(\x_n))\|_2.
\end{aligned}
\end{equation}
\end{proposition}

\begin{proof}
The equalities follow by direct substitution and the norm-preserving property of $\phi$.
\end{proof}

\paragraph{Consequence (worst-case impossibility from $Z$ alone).}
Proposition~\ref{prop:latent_indistinguishability} shows that the same observed uploaded latent set $Z$ can arise from multiple distinct private datasets that are equally consistent with autoencoder-style reconstruction objectives. Therefore, without additional assumptions linking the coordinates of $\mathcal X$ to the semantic structure, no attacker that only observes $Z$ can guarantee recovery of the original private records $\{\x_n\}$.
%\textcolor{red}{[C: $Z$ is not defined. Is $Z$ an arbitrary latent code?]}

\paragraph{Why latent-only flipping is weaker than white-box inversion.}
In standard FL, the server has a white-box copy of the victim reconstruction map $T(\x)=D(E(\x))$ and can directly optimize inputs to reduce the \emph{input-space} objective $r(\x)=\|\x-T(\x)\|_2^2$, which tightly couples the attack to the victim model and the true input domain.
In GCA, the server does not observe $D$ (or gradients through it) and never sees any $\x_n$; the latent-only objective \eqref{eq:app_flip_train_obj} provides constraints only in latent space and only on the finite set $Z$.
As a result, the attacker must infer an input-space mapping $\tilde D$ from latents alone, and the problem is fundamentally underdetermined: many distinct input-space reconstructions are compatible with the same observed latents and can achieve similarly low latent reconstruction loss.
This underdetermination prevents a latent-only surrogate from guaranteeing faithful extraction; the comparative strength of the resulting surrogate and white-box attacks is therefore determined empirically.

\clearpage
\newpage
\section{Additional Experimental Results and Ablations}
\label{app:additional-experiments}

All method implementations, shared settings, data partitions, model and $K$ selection, attack procedures, and communication accounting follow Appendix~\ref{app:implementation-details}.

\subsection{GCA-GMM Component-Count Selection}
\label{app:iid-k-selection}

For each seed and $K$, GCA accuracy is the maximum over reconstruction-epoch test evaluations. $K^\star$ maximizes the mean of the three seed-level values, with exact ties resolved in favor of smaller $K$.

\begin{papertableblock}
\papertablecaption{Selected GCA-GMM component count for each dataset.}
\label{tab:selected-gca-k}
\papertablesize
\begin{tabular}{lrrrrrrr}
\toprule
 & Academic & Adult & Bank & Credit & MAGIC & CIFAR-10 & FMNIST \\
\midrule
$K^\star$ & 10 & 10 & 10 & 20 & 80 & 20 & 10 \\
\bottomrule
\end{tabular}

\end{papertableblock}

\subsection{Accuracy by Component Count}
\label{app:iid-k-sweeps}

\begin{papertableblock}
\papertablecaption{GCA-GMM accuracy by component count in the IID setting. Values are test accuracy in percent. The two highest unrounded means for each dataset are bold; exact ties are resolved toward smaller $K$.}
\label{tab:iid-gmm-k-sweep}
\papertablesize
\begin{tabular}{lrrrr}
\toprule
Dataset & $K=10$ & $K=20$ & $K=40$ & $K=80$ \\
\midrule
Academic & \textbf{64.92 $\pm$ 0.01} & 64.91 $\pm$ 0.01 & \textbf{64.92 $\pm$ 0.01} & 64.91 $\pm$ 0.01 \\
Adult & \textbf{66.87 $\pm$ 0.02} & \textbf{66.87 $\pm$ 0.02} & 66.87 $\pm$ 0.02 & 66.87 $\pm$ 0.02 \\
Bank & \textbf{51.18 $\pm$ 0.06} & \textbf{51.10 $\pm$ 0.07} & 51.02 $\pm$ 0.03 & 50.95 $\pm$ 0.15 \\
Credit & 82.31 $\pm$ 0.18 & \textbf{82.40 $\pm$ 0.09} & \textbf{82.37 $\pm$ 0.06} & 82.02 $\pm$ 0.04 \\
MAGIC & \textbf{69.90 $\pm$ 0.35} & 69.75 $\pm$ 0.38 & 69.80 $\pm$ 0.28 & \textbf{69.97 $\pm$ 0.18} \\
CIFAR-10 & \textbf{55.39 $\pm$ 5.29} & \textbf{55.62 $\pm$ 5.09} & 54.39 $\pm$ 5.50 & 55.29 $\pm$ 6.05 \\
FMNIST & \textbf{78.61 $\pm$ 5.06} & \textbf{78.49 $\pm$ 6.23} & 77.79 $\pm$ 5.34 & 77.30 $\pm$ 5.29 \\
\bottomrule
\end{tabular}

\end{papertableblock}

\subsection{Alignment-Phase Weighting Ablation}
\label{app:uniform-weight-gmm}

To isolate the effect of inverse-support weighting on the alignment objective, we compare inverse weighting with uniform weighting, where all nonempty centroid weights are set to $w_k=1$. Both configurations use the inverse-weighted model's selected $K^\star$ and otherwise share the same data partitions, latent uploads, GMM fitting procedure, reconstruction--alignment schedule, and optimization settings.

\begin{papertableblock}
\papertablecaption{Peak alignment-epoch accuracy for inverse and uniform weighting at the selected $K^\star$. For each seed, the highest value is selected from the 500 alignment-epoch evaluations over 100 rounds. Values are percentages (mean $\pm$ standard deviation); the difference and relative change compare inverse weighting with uniform weighting.}
\label{tab:iid-gmm-alignment-weighting}
\papertablesize
\begin{tabular}{lrrrrr}
\toprule
Dataset & $K^\star$ & Inverse & Uniform & Difference & Relative change \\
\midrule
Academic & 10 & 65.29 $\pm$ 0.39 & 65.12 $\pm$ 0.42 & +0.18 pp & +0.27\% \\
Adult & 10 & 68.80 $\pm$ 0.92 & 69.20 $\pm$ 0.57 & -0.40 pp & -0.57\% \\
Bank & 10 & 56.81 $\pm$ 2.29 & 56.48 $\pm$ 1.36 & +0.34 pp & +0.59\% \\
Credit & 20 & 86.25 $\pm$ 0.74 & 84.89 $\pm$ 0.23 & +1.36 pp & +1.60\% \\
MAGIC & 80 & 72.93 $\pm$ 1.72 & 74.21 $\pm$ 2.35 & -1.28 pp & -1.72\% \\
CIFAR-10 & 20 & 67.40 $\pm$ 10.19 & 65.15 $\pm$ 8.30 & +2.25 pp & +3.45\% \\
FMNIST & 10 & 81.97 $\pm$ 5.67 & 81.82 $\pm$ 5.83 & +0.15 pp & +0.18\% \\
\bottomrule
\end{tabular}

\end{papertableblock}

Because centroid weights directly affect the alignment objective, we evaluate their contribution using alignment-epoch accuracy for one fixed client, used consistently across each paired inverse- and uniform-weighting configuration. This diagnostic is separate from the primary all-client reconstruction-accuracy evaluation.

Inverse weighting achieves a higher mean accuracy on five of seven datasets and in 14 of 21 seed-level comparisons. Its unweighted average across the seven datasets is $71.35\%$, compared with $70.98\%$ for uniform weighting, corresponding to an improvement of $0.37$ percentage points. Relative changes range from $-1.72\%$ to $+3.45\%$ across all datasets and from $+0.18\%$ to $+3.45\%$ among the five datasets favoring inverse weighting. These results indicate that inverse-support weighting improves peak alignment performance overall, although uniform weighting performs better on Adult and MAGIC.

\subsection{IID Baseline Comparison}
\label{app:iid-baseline-results}

\begin{papertableblock}
\papertablecaption{Mean $\pm$ standard deviation for the IID comparison. The two highest decentralized testing accuracies for each dataset are bold; Single-Client and Centralized (in italics) are excluded from this ranking.}
\label{tab:iid-expanded-baselines}
\papertablesize
\begin{tabular}{lrrrrrrr}
\toprule
Dataset & GCA & \textit{Single} & FedAvg & FedProx & FedNova & DP-FedAvg & \textit{Centralized} \\
\midrule
Academic & \textbf{64.92 $\pm$ 0.01} & \textit{64.54 $\pm$ 0.03} & 64.88 $\pm$ 0.02 & 64.90 $\pm$ 0.02 & 64.88 $\pm$ 0.02 & \textbf{65.86 $\pm$ 0.27} & \textit{65.95 $\pm$ 0.49} \\
Adult & \textbf{66.87 $\pm$ 0.02} & \textit{63.25 $\pm$ 0.24} & 65.84 $\pm$ 0.16 & 66.07 $\pm$ 0.03 & 65.84 $\pm$ 0.16 & \textbf{66.82 $\pm$ 0.08} & \textit{69.47 $\pm$ 0.28} \\
Bank & 51.18 $\pm$ 0.06 & \textit{50.74 $\pm$ 0.06} & 51.65 $\pm$ 0.52 & 51.29 $\pm$ 0.14 & \textbf{51.71 $\pm$ 0.50} & \textbf{51.92 $\pm$ 0.22} & \textit{52.83 $\pm$ 0.30} \\
Credit & \textbf{82.40 $\pm$ 0.09} & \textit{80.06 $\pm$ 0.18} & 82.10 $\pm$ 0.85 & 81.39 $\pm$ 0.01 & \textbf{82.22 $\pm$ 1.07} & 81.00 $\pm$ 0.32 & \textit{84.08 $\pm$ 1.12} \\
MAGIC & \textbf{69.97 $\pm$ 0.18} & \textit{68.75 $\pm$ 0.06} & 69.86 $\pm$ 0.73 & 69.86 $\pm$ 0.65 & \textbf{69.94 $\pm$ 0.75} & 69.88 $\pm$ 0.98 & \textit{73.36 $\pm$ 0.18} \\
CIFAR-10 & 55.62 $\pm$ 5.09 & \textit{54.09 $\pm$ 7.72} & 55.86 $\pm$ 6.20 & \textbf{56.56 $\pm$ 6.37} & 55.94 $\pm$ 6.44 & \textbf{63.71 $\pm$ 11.41} & \textit{53.72 $\pm$ 9.86} \\
FMNIST & \textbf{78.61 $\pm$ 5.06} & \textit{75.02 $\pm$ 5.42} & 74.33 $\pm$ 6.54 & 74.51 $\pm$ 6.99 & 74.32 $\pm$ 6.53 & \textbf{77.40 $\pm$ 2.60} & \textit{75.03 $\pm$ 4.50} \\
\bottomrule
\end{tabular}

\end{papertableblock}

GCA improves on Single-Client for all seven dataset means and on FedAvg for five of seven. Against FedProx, FedNova, and DP-FedAvg it records 5/2, 5/2, and 4/3 win/loss counts, respectively. The corresponding mixture of outcomes supports comparable decentralized accuracy rather than uniform dominance.

\subsection{Clustering-Backend Ablations}
\label{app:clustering-backend-ablations}
To evaluate sensitivity to the clustering backend, we repeat the IID component-count sweep using K-means, soft K-means, and mini-batch K-means. All other data, training, weighting, and reporting settings match the GCA-GMM evaluation.

\begin{papertableblock}
\papertablecaption{GCA--K-means accuracy by component count in the IID setting. Values are test accuracy in percent (mean $\pm$ standard deviation over three seeds). The two highest unrounded means for each dataset are bold; exact ties are resolved toward smaller $K$.}
\label{tab:iid-kmeans-k-sweep}
\papertablesize
\begin{tabular}{lrrrr}
\toprule
Dataset & $K=10$ & $K=20$ & $K=40$ & $K=80$ \\
\midrule
Academic & 64.91 $\pm$ 0.01 & 64.91 $\pm$ 0.01 & \textbf{64.92 $\pm$ 0.00} & \textbf{64.92 $\pm$ 0.00} \\
Adult & \textbf{66.87 $\pm$ 0.02} & \textbf{66.87 $\pm$ 0.02} & 66.87 $\pm$ 0.02 & 66.87 $\pm$ 0.02 \\
Bank & \textbf{51.13 $\pm$ 0.15} & \textbf{51.12 $\pm$ 0.03} & 50.91 $\pm$ 0.04 & 50.85 $\pm$ 0.15 \\
Credit & \textbf{82.56 $\pm$ 0.10} & 82.47 $\pm$ 0.27 & \textbf{82.48 $\pm$ 0.10} & 82.42 $\pm$ 0.20 \\
MAGIC & \textbf{70.07 $\pm$ 0.33} & 69.93 $\pm$ 0.37 & 69.74 $\pm$ 0.35 & \textbf{69.95 $\pm$ 0.28} \\
CIFAR-10 & \textbf{55.08 $\pm$ 5.82} & 55.04 $\pm$ 6.36 & \textbf{55.37 $\pm$ 6.53} & 55.05 $\pm$ 6.27 \\
FMNIST & \textbf{79.29 $\pm$ 5.05} & \textbf{79.29 $\pm$ 5.53} & 77.85 $\pm$ 4.93 & 77.43 $\pm$ 4.94 \\
\bottomrule
\end{tabular}

\end{papertableblock}

\begin{papertableblock}
\papertablecaption{GCA--Soft-K-means accuracy by component count in the IID setting. Values are test accuracy in percent (mean $\pm$ standard deviation over three seeds). The two highest unrounded means for each dataset are bold; exact ties are resolved toward smaller $K$.}
\label{tab:iid-soft-kmeans-k-sweep}
\papertablesize
\begin{tabular}{lrrrr}
\toprule
Dataset & $K=10$ & $K=20$ & $K=40$ & $K=80$ \\
\midrule
Academic & 64.91 $\pm$ 0.01 & 64.91 $\pm$ 0.02 & \textbf{64.91 $\pm$ 0.01} & \textbf{64.91 $\pm$ 0.01} \\
Adult & \textbf{66.87 $\pm$ 0.02} & \textbf{66.87 $\pm$ 0.02} & 66.87 $\pm$ 0.02 & 66.87 $\pm$ 0.02 \\
Bank & \textbf{51.16 $\pm$ 0.09} & 50.97 $\pm$ 0.16 & \textbf{51.06 $\pm$ 0.09} & 50.86 $\pm$ 0.03 \\
Credit & \textbf{82.53 $\pm$ 0.18} & 82.42 $\pm$ 0.06 & \textbf{82.43 $\pm$ 0.28} & 82.23 $\pm$ 0.02 \\
MAGIC & \textbf{70.17 $\pm$ 0.40} & \textbf{69.89 $\pm$ 0.24} & 69.80 $\pm$ 0.36 & 69.88 $\pm$ 0.35 \\
CIFAR-10 & \textbf{55.24 $\pm$ 6.37} & 54.96 $\pm$ 6.02 & \textbf{54.97 $\pm$ 6.31} & 54.51 $\pm$ 5.56 \\
FMNIST & \textbf{79.38 $\pm$ 4.86} & \textbf{79.24 $\pm$ 5.48} & 77.66 $\pm$ 5.10 & 77.26 $\pm$ 5.20 \\
\bottomrule
\end{tabular}

\end{papertableblock}

\begin{papertableblock}
\papertablecaption{GCA--MiniBatch-K-means accuracy by component count in the IID setting. Values are test accuracy in percent (mean $\pm$ standard deviation over three seeds). The two highest unrounded means for each dataset are bold; exact ties are resolved toward smaller $K$.}
\label{tab:iid-minibatch-kmeans-k-sweep}
\papertablesize
\begin{tabular}{lrrrr}
\toprule
Dataset & $K=10$ & $K=20$ & $K=40$ & $K=80$ \\
\midrule
Academic & \textbf{64.92 $\pm$ 0.01} & 64.90 $\pm$ 0.02 & \textbf{64.92 $\pm$ 0.03} & 64.92 $\pm$ 0.01 \\
Adult & \textbf{66.87 $\pm$ 0.02} & \textbf{66.87 $\pm$ 0.02} & 66.87 $\pm$ 0.02 & 66.87 $\pm$ 0.02 \\
Bank & 51.07 $\pm$ 0.12 & \textbf{51.16 $\pm$ 0.34} & \textbf{51.11 $\pm$ 0.03} & 51.04 $\pm$ 0.14 \\
Credit & \textbf{82.45 $\pm$ 0.19} & \textbf{82.43 $\pm$ 0.19} & 82.37 $\pm$ 0.18 & 82.34 $\pm$ 0.07 \\
MAGIC & \textbf{70.36 $\pm$ 0.43} & \textbf{70.40 $\pm$ 0.55} & 70.23 $\pm$ 0.52 & 69.76 $\pm$ 0.36 \\
CIFAR-10 & 54.77 $\pm$ 5.66 & 55.31 $\pm$ 6.28 & \textbf{55.39 $\pm$ 5.42} & \textbf{55.52 $\pm$ 6.72} \\
FMNIST & \textbf{78.83 $\pm$ 4.45} & \textbf{78.32 $\pm$ 6.18} & 77.63 $\pm$ 5.04 & 77.40 $\pm$ 4.82 \\
\bottomrule
\end{tabular}

\end{papertableblock}
Accuracy varies only modestly across the evaluated component counts and
clustering backends. This indicates limited sensitivity to the specific
clustering implementation rather than uniform superiority of GMM.
We use GMM as the primary backend because its soft responsibilities
provide effective support counts directly for inverse-support weighting.

\newpage
\subsection{Server-Side Extraction Results}
\label{app:attack-results}

Let $\mathcal{T}_i$ be the complete training shard of the one evaluated
client and let $\mathcal{R}_i=\{\hat{\mathbf{x}}_j\}_{j=1}^{100}$ be its
100 reconstructed attack outputs. For every target
$\mathbf{x}\in\mathcal{T}_i$, we first select the output with minimum
MSE. Cosine similarity is then evaluated against that same selected
output, rather than choosing a different output that maximizes cosine.
This target-oriented metric measures how well the reconstructed output bank covers the complete client shard. The same reconstructed output may be nearest to multiple training targets, so the metric is not a one-to-one record-recovery measure.

A fixed reference bank $\mathcal{H}_i$ containing 100 normal records excluded from training provides the dataset- and client-specific baseline.
Let
$\overline{\mathrm{MSE}}(\mathcal A,\mathcal B)$ denote the mean, over records in $\mathcal A$, of MSE to the nearest record in $\mathcal B$. Let $\overline{\cos}_{\mathrm{MSE}}(\mathcal A,\mathcal B)$ denote mean cosine similarity to those same MSE-selected records. The two reported leakage metrics are
\[
\mathrm{NTMSE}=
\frac{\overline{\mathrm{MSE}}(\mathcal{T}_i,\mathcal{R}_i)}
{\overline{\mathrm{MSE}}(\mathcal{T}_i,\mathcal{H}_i)},\qquad
\Delta\cos=
\overline{\cos}_{\mathrm{MSE}}(\mathcal{T}_i,\mathcal{R}_i)-
\overline{\cos}_{\mathrm{MSE}}(\mathcal{T}_i,\mathcal{H}_i).
\]
NTMSE (normalized target MSE) is a distance ratio, so values above one indicate that reconstructed outputs are farther from training targets than the held-out reference bank. Excess cosine, $\Delta\cos$, is the corresponding similarity difference; negative values mean the reconstructed outputs have lower cosine similarity to the targets than the held-out bank. Consequently, higher NTMSE and lower $\Delta\cos$ indicate less leakage. We average records within the evaluated client and then report the mean and standard deviation over the three dataset seeds.

\begin{papertableblock}
\papertablecaption{Per-dataset extraction results for one evaluated client per seed (mean $\pm$ standard deviation over three seeds).
For each training record, the nearest of 100 reconstructed outputs is selected by MSE, and cosine similarity uses that same output. Higher NTMSE and lower excess cosine similarity ($\Delta\cos$) indicate lower target resemblance under the evaluated attack; bold marks the lowest-target-resemblance result per dataset and metric.}
\label{tab:expanded-attack-primary}
\papertablesize
\begin{tabular}{@{}llrr@{}}
\toprule
Dataset & Method & NTMSE $\uparrow$ & $\Delta$cos $\downarrow$ \\
\midrule
Academic & GCA-GMM & 1.838$\pm$0.138 & -0.563$\pm$0.131 \\
Academic & FedAvg & 0.921$\pm$0.027 & -0.007$\pm$0.009 \\
Academic & FedProx & 0.958$\pm$0.117 & -0.029$\pm$0.019 \\
Academic & FedNova & 0.915$\pm$0.035 & -0.011$\pm$0.005 \\
Academic & DP-FedAvg & \textbf{276.171$\pm$182.022} & \textbf{-0.724$\pm$0.003} \\
\addlinespace
Adult & GCA-GMM & \textbf{2.216$\pm$0.140} & \textbf{-0.702$\pm$0.056} \\
Adult & FedAvg & 1.298$\pm$0.045 & -0.098$\pm$0.006 \\
Adult & FedProx & 1.291$\pm$0.026 & -0.096$\pm$0.004 \\
Adult & FedNova & 1.265$\pm$0.025 & -0.092$\pm$0.005 \\
Adult & DP-FedAvg & 2.060$\pm$0.133 & -0.563$\pm$0.160 \\
\addlinespace
Bank & GCA-GMM & \textbf{1.584$\pm$0.065} & \textbf{-0.612$\pm$0.019} \\
Bank & FedAvg & 0.935$\pm$0.037 & 0.029$\pm$0.011 \\
Bank & FedProx & 0.916$\pm$0.063 & 0.028$\pm$0.026 \\
Bank & FedNova & 0.955$\pm$0.065 & 0.025$\pm$0.018 \\
Bank & DP-FedAvg & 1.149$\pm$0.175 & -0.301$\pm$0.225 \\
\addlinespace
Credit & GCA-GMM & \textbf{1.352$\pm$0.021} & -0.559$\pm$0.055 \\
Credit & FedAvg & 1.103$\pm$0.017 & -0.056$\pm$0.005 \\
Credit & FedProx & 1.346$\pm$0.029 & \textbf{-0.584$\pm$0.009} \\
Credit & FedNova & 1.109$\pm$0.020 & -0.059$\pm$0.011 \\
Credit & DP-FedAvg & 1.100$\pm$0.019 & -0.048$\pm$0.010 \\
\addlinespace
MAGIC & GCA-GMM & 4.931$\pm$1.023 & -0.493$\pm$0.202 \\
MAGIC & FedAvg & 1.688$\pm$0.087 & -0.033$\pm$0.009 \\
MAGIC & FedProx & 1.877$\pm$0.203 & -0.047$\pm$0.031 \\
MAGIC & FedNova & 1.741$\pm$0.170 & -0.048$\pm$0.030 \\
MAGIC & DP-FedAvg & \textbf{6.047$\pm$0.373} & \textbf{-0.919$\pm$0.091} \\
\addlinespace
CIFAR-10 & GCA-GMM & 2.139$\pm$0.798 & \textbf{-0.540$\pm$0.169} \\
CIFAR-10 & FedAvg & 1.584$\pm$0.233 & -0.457$\pm$0.047 \\
CIFAR-10 & FedProx & 1.600$\pm$0.204 & -0.454$\pm$0.035 \\
CIFAR-10 & FedNova & 1.586$\pm$0.211 & -0.438$\pm$0.048 \\
CIFAR-10 & DP-FedAvg & \textbf{5.502$\pm$1.130} & -0.318$\pm$0.104 \\
\addlinespace
FMNIST & GCA-GMM & \textbf{13.940$\pm$1.236} & \textbf{-1.003$\pm$0.101} \\
FMNIST & FedAvg & 5.007$\pm$0.404 & -0.674$\pm$0.019 \\
FMNIST & FedProx & 5.570$\pm$0.562 & -0.733$\pm$0.054 \\
FMNIST & FedNova & 4.987$\pm$0.402 & -0.663$\pm$0.011 \\
FMNIST & DP-FedAvg & 9.573$\pm$1.488 & -0.647$\pm$0.106 \\
\bottomrule
\end{tabular}

\end{papertableblock}

GCA-GMM has higher NTMSE than FedAvg, FedProx, and FedNova in all 21 comparisons. Its excess cosine similarity is lower in 20 of 21: all seven datasets against FedAvg and FedNova, and six of seven against FedProx. Against DP-FedAvg, GCA-GMM has higher NTMSE on four datasets and lower excess cosine similarity on five. The two metrics therefore consistently favor GCA-GMM over FedAvg, FedProx, and FedNova, while the comparison with DP-FedAvg is mixed.

\subsection{Deterministic Communication per Round}
\label{app:efficiency-results}

\begin{papertableblock}
\papertablecaption{Inputs to the deterministic communication accounting. $C$ is the client count, $U$ is the total uploaded codes per round, $D_z$ is the flattened latent dimension, and $P$ is the autoencoder parameter count. Payload columns are MiB per round; FL total applies identically to FedAvg, FedProx, FedNova, and DP-FedAvg.}
\label{tab:communication-inputs}
\papertablesize
\begin{tabular}{lrrrrrr}
\toprule
Dataset & $C$ & $U$ & $D_z$ & $P$ & GCA up & FL total \\
\midrule
Academic & 10 & 390 & 16 & 46,758 & 0.024 & 3.567 \\
Adult & 20 & 2,460 & 16 & 42,654 & 0.150 & 6.508 \\
Bank & 20 & 3,960 & 16 & 43,167 & 0.242 & 6.587 \\
Credit & 50 & 28,350 & 16 & 50,349 & 1.730 & 19.207 \\
MAGIC & 10 & 1,410 & 16 & 38,037 & 0.086 & 2.902 \\
CIFAR-10 & 20 & 500 & 1,024 & 339,219 & 1.953 & 51.761 \\
FMNIST & 20 & 600 & 784 & 334,609 & 1.794 & 51.057 \\
\bottomrule
\end{tabular}

\end{papertableblock}

\begin{papertableblock}
\papertablecaption{Deterministic total communication per round in MiB. GCA columns include the latent-code uplink and the centroid-and-count broadcast to every client. FL denotes the identical full-model upload-and-broadcast payload of FedAvg, FedProx, FedNova, and DP-FedAvg.}
\label{tab:communication-per-round}
\papertablesize
\begin{tabular}{lrrrrr}
\toprule
Dataset & GCA $K=10$ & GCA $K=20$ & GCA $K=40$ & GCA $K=80$ & FL \\
\midrule
Academic & 0.030 & 0.037 & 0.050 & 0.076 & 3.567 \\
Adult & 0.163 & 0.176 & 0.202 & 0.254 & 6.508 \\
Bank & 0.255 & 0.268 & 0.294 & 0.345 & 6.587 \\
Credit & 1.763 & 1.795 & 1.860 & 1.990 & 19.207 \\
MAGIC & 0.093 & 0.099 & 0.112 & 0.138 & 2.902 \\
CIFAR-10 & 2.735 & 3.517 & 5.081 & 8.209 & 51.761 \\
FMNIST & 2.393 & 2.992 & 4.190 & 6.586 & 51.057 \\
\bottomrule
\end{tabular}

\end{papertableblock}

For all evaluated datasets and all four tested component counts, GCA transmits less data per round than the parameter-sharing baselines. Across the tabular datasets, the reduction ranges from 89.64\% (Credit, $K=80$) to 99.15\% (Academic, $K=10$). With the two-downsampling vision architecture, CIFAR-10 uses 2.735--8.209 MiB per GCA round versus 51.761 MiB for FL, and Fashion-MNIST uses 2.393--6.586 MiB versus 51.057 MiB. Even at $K=80$, these are reductions of 84.14\% and 87.10\%, respectively.

\subsection{Non-IID Results}
\label{app:non-iid-results}

The non-IID experiments introduce both feature-distribution and sample-count skew. Samples are grouped using their feature descriptors, and a Dirichlet distribution with $\alpha=0.1$ assigns each feature group concentrated client preferences. Long-tail client quotas additionally impose a target 10:1 maximum-to-minimum sample-count ratio. Clients therefore observe different feature distributions and unequal quantities of normal training data; the complete partition protocol is given in Appendix~\ref{app:data-partitioning}.

\begin{papertableblock}
\papertablecaption{GCA-GMM $K$ sweep in the non-IID setting. Values are test accuracy in percent.}
\label{tab:noniid-gmm-k-sweep}
\papertablesize
\begin{tabular}{lrrrr}
\toprule
Dataset & $K=10$ & $K=20$ & $K=40$ & $K=80$ \\
\midrule
Academic & 63.36 $\pm$ 0.76 & 63.36 $\pm$ 0.76 & 63.36 $\pm$ 0.76 & 63.36 $\pm$ 0.76 \\
Adult & 43.41 $\pm$ 2.69 & 43.41 $\pm$ 2.69 & 43.41 $\pm$ 2.69 & 43.41 $\pm$ 2.69 \\
Bank & 50.46 $\pm$ 0.23 & 50.46 $\pm$ 0.23 & 50.46 $\pm$ 0.23 & 50.46 $\pm$ 0.23 \\
Credit & 62.33 $\pm$ 0.76 & 62.33 $\pm$ 0.76 & 62.33 $\pm$ 0.76 & 62.33 $\pm$ 0.76 \\
MAGIC & 61.75 $\pm$ 0.58 & 61.75 $\pm$ 0.58 & 61.75 $\pm$ 0.58 & 61.75 $\pm$ 0.58 \\
\bottomrule
\end{tabular}

\end{papertableblock}

\begin{papertableblock}
\papertablecaption{Selected GCA-GMM ($K^\star=10$) and FL baselines in the non-IID setting. Values are test accuracy in percent.}
\label{tab:noniid-results}
\papertablesize
\begin{tabular}{lrrrrr}
\toprule
Dataset & GCA & FedAvg & FedProx & FedNova & DP-FedAvg \\
\midrule
Academic & 63.36 $\pm$ 0.76 & 62.24 $\pm$ 1.20 & 62.26 $\pm$ 1.20 & 62.23 $\pm$ 1.17 & 63.88 $\pm$ 0.40 \\
Adult & 43.41 $\pm$ 2.69 & 38.88 $\pm$ 0.87 & 39.26 $\pm$ 1.51 & 38.88 $\pm$ 0.87 & 39.66 $\pm$ 1.26 \\
Bank & 50.46 $\pm$ 0.23 & 50.39 $\pm$ 0.28 & 50.42 $\pm$ 0.25 & 50.40 $\pm$ 0.29 & 50.47 $\pm$ 0.13 \\
Credit & 62.33 $\pm$ 0.76 & 61.19 $\pm$ 0.89 & 60.96 $\pm$ 0.74 & 61.19 $\pm$ 0.89 & 61.19 $\pm$ 0.89 \\
MAGIC & 61.75 $\pm$ 0.58 & 62.39 $\pm$ 0.45 & 62.45 $\pm$ 0.69 & 61.68 $\pm$ 1.03 & 59.01 $\pm$ 1.71 \\
\bottomrule
\end{tabular}

\end{papertableblock}

All tested $K$ values attain the same reported reconstruction-epoch peak, so the stated tie rule selects $K^\star=10$. Compared with FedAvg, GCA improves Academic, Adult, Bank, and Credit and trails on MAGIC, with per-dataset relative changes from $-1.03\%$ to $11.66\%$.
The reported $2.26\%$ is the relative difference between the five-dataset macro-averages:
\[
\frac{
\overline{A}_{\mathrm{GCA}}
-
\overline{A}_{\mathrm{FedAvg}}
}{
\overline{A}_{\mathrm{FedAvg}}
}
\times100.
\]

\stopcontents[appendix]

\end{document}